\pdfoutput=1
\documentclass{article}

\usepackage[final, nonatbib]{neurips_2022}

\usepackage[utf8]{inputenc} % allow utf-8 input
\usepackage[T1]{fontenc}    % use 8-bit T1 fonts
\usepackage[pagebackref=true]{hyperref}       % hyperlinks
\renewcommand*\backref[1]{\ifx#1\relax \else (Cited on #1) \fi}
\usepackage{url}            % simple URL typesetting
\usepackage{booktabs}       % professional-quality tables
\usepackage{amsfonts}       % blackboard math symbols
\usepackage{nicefrac}       % compact symbols for 1/2, etc.
\usepackage{microtype}      % microtypography
\usepackage{xcolor}         % colors

\usepackage{bm}
\usepackage{amsmath}
\usepackage{amsthm}
\usepackage{amssymb}
\usepackage{bbm}
\usepackage{float}
\usepackage{makecell}
\usepackage{adjustbox}
\usepackage{multirow}
\usepackage{wrapfig}
\usepackage{subcaption}
\usepackage{rotating}

\newtheorem{thm}{Theorem}
\newtheorem{defn}{Definition}

\newtheorem{property}{Property}
\newtheorem{corollary}{Corollary}
\newtheorem{lemma}{Lemma}

\newcommand{\cutparagraphup}{\vspace*{-0.1in}}

\title{Pure Transformers are Powerful Graph Learners}

\author{%
    \parbox{0.8\linewidth}{
    \centering
    \vspace{-0.1cm}
    Jinwoo Kim$^1$\thanks{Work done during an internship at LG AI Research.}\hspace{0.4em}\hspace{1em}
    Tien Dat Nguyen$^1$\hspace{1em}
    Seonwoo Min$^2$\hspace{1em}
    Sungjun Cho$^2$\hspace{1em}
    Moontae Lee$^{2,3}$\hspace{1em}
    Honglak Lee$^2$\thanks{Equal corresponding authors.}\hspace{1em}
    Seunghoon Hong$^{1,2}$\footnotemark[2]\hspace{0.4em}
    }\vspace{0.12cm}\\
    $^1$KAIST
    $^2$LG AI Research
    $^3$University of Illinois Chicago%\\
    \vspace{-0.1cm}
}

\begin{document}

\maketitle

\vspace{-0.2cm}
\begin{abstract}
\vspace{-0.2cm}
  We show that standard Transformers without graph-specific modifications can lead to promising results in graph learning both in theory and practice.
  Given a graph, we simply treat all nodes and edges as independent tokens, augment them with token embeddings, and feed them to a Transformer.
  With an appropriate choice of token embeddings, we prove that this approach is theoretically at least as expressive as an invariant graph network (2-IGN) composed of equivariant linear layers, which is already more expressive than all message-passing Graph Neural Networks (GNN).
  When trained on a large-scale graph dataset (PCQM4Mv2), our method coined \textbf{Tokenized Graph Transformer (TokenGT)} achieves significantly better results compared to GNN baselines and competitive results compared to Transformer variants with sophisticated graph-specific inductive bias.
  Our implementation is available at \url{https://github.com/jw9730/tokengt}.
\end{abstract}

\section{Introduction}\label{sec:introduction}

In recent years, Transformer~\cite{vaswani2017attention} has served as a versatile architecture in a broad class of machine learning problems, such as natural language processing~\cite{devlin2019bert, brown2020language}, computer vision~\cite{dosovitskiy2021animage}, and reinforcement learning~\cite{chen2021decision}, to name a few.
It is because the fully-attentional structure of Transformer is general and powerful enough to take, process, and relate inputs and outputs of arbitrary structures, eliminating a need for data- and task-specific inductive bias to be baked into the network architecture.
Combined with large-scale training, it opens up a new chapter for building a versatile model that can solve a wide range of problems involving diverse data modalities and even a mixture of modalities~\cite{andrew2021perceiver, andrew2021perceiverio, reed2022ageneralist}. 

In graph learning domain, inspired by the breakthroughs, multiple works tried combining self-attention into graph neural network (GNN) architecture where message passing was previously dominant~\cite{min2022transformer}.
As global self-attention across nodes cannot reflect the graph structure, however, these methods introduce \emph{graph-specific} architectural modifications.
This includes restricting self-attention to local neighborhoods~\cite{velikovic2018graph, nguyen2022universal, dwivedi2020a}, using global self-attention in conjunction with message-passing GNN~\cite{rong2020self, lin2021mesh, kim2021transformers}, and injecting edge information into global self-attention via attention bias~\cite{wang2019self, ying2021do, hussain2021edge, park2022grpe}.
Despite decent performance, such modifications can be a limiting constraint in terms of versatility, especially considering future integration to multi-task and multi-modal general-purpose attentional architectures~\cite{andrew2021perceiver}.
In addition, deviating from pure self-attention, these methods may inherit the issues of message-passing such as oversmoothing~\cite{li2018deeper, cai2020a, oono2020graph}, and become incompatible with useful engineering techniques \emph{e.g.}, linear attention~\cite{tay2020efficient} developed for standard self-attention.

Instead, we explore the opposite direction of \emph{applying a standard Transformer directly for graphs}.
For this, we treat all nodes and edges as independent tokens, augment them with appropriate token-wise embeddings, and feed the tokens as input to the standard Transformer.
The model operates identically to Transformers used in language and vision; each node or edge is treated as a token, identical to the words in a sentence or patches of an image~\cite{vaswani2017attention, dosovitskiy2021animage}.
Perhaps surprisingly, we show that this simple approach yields a powerful graph learner both in theory and practice.

As a key theoretical result, we prove that with appropriate token-wise embeddings, self-attention over the node and edge tokens can approximate \emph{any} permutation equivariant linear operator on a graph~\cite{maron2019invariant}.
Remarkably, we show that a very simple choice of embedding composed of \emph{node identifiers} and \emph{type identifiers} is sufficient for accurate approximation.
This provides a solid theoretical guarantee that, with the embeddings and enough attention heads, a Transformer is at least as expressive as a second-order invariant graph network (2-IGN)~\cite{maron2019invariant,kim2021transformers}, which is already more expressive than all message-passing GNNs~\cite{gilmer2017neural}.
This also immediately grants the model with the expressive power at least as good as the 2-dimensional Weisfeiler-Lehman (WL) graph isomorphism test~\cite{maron2019provably}, which is often sufficient for real-world graph data~\cite{zopf20221_wl}.
We further extend our theoretical result to \emph{hypergraphs} with order-$k$ hyperedges, showing that a Transformer with order-$k$ generalized token embeddings is at least as expressive as $k$-IGN and, consequently $k$-WL test.

We test our model, named Tokenized Graph Transformer (TokenGT), mainly on the PCQM4Mv2 large-scale quantum chemical property prediction dataset containing 3.7M molecular graphs~\cite{hu2021ogb}.
Even though TokenGT involves minimal graph-specific architectural modifications, it performs significantly better than all GNN baselines, showing that the advantages of Transformer architecture combined with large-scale training surpass the benefit of hard inductive bias of GNNs.
Furthermore, TokenGT achieves competitive performance compared to Transformer variants with strong graph-specific modifications~\cite{ying2021do, hussain2021edge, park2022grpe}.
Finally, we demonstrate that TokenGT can naturally utilize efficient approximations in Transformers in contrast to these variants, using kernel attention~\cite{choromanski2021rethinking} that enables linear computation cost without much degradation in performance.

\section{Tokenized Graph Transformer (TokenGT)}\label{sec:methods}

\begin{figure}[!t]
    \vspace{-0.2em}
    \centering
    \includegraphics[width=0.99\textwidth]{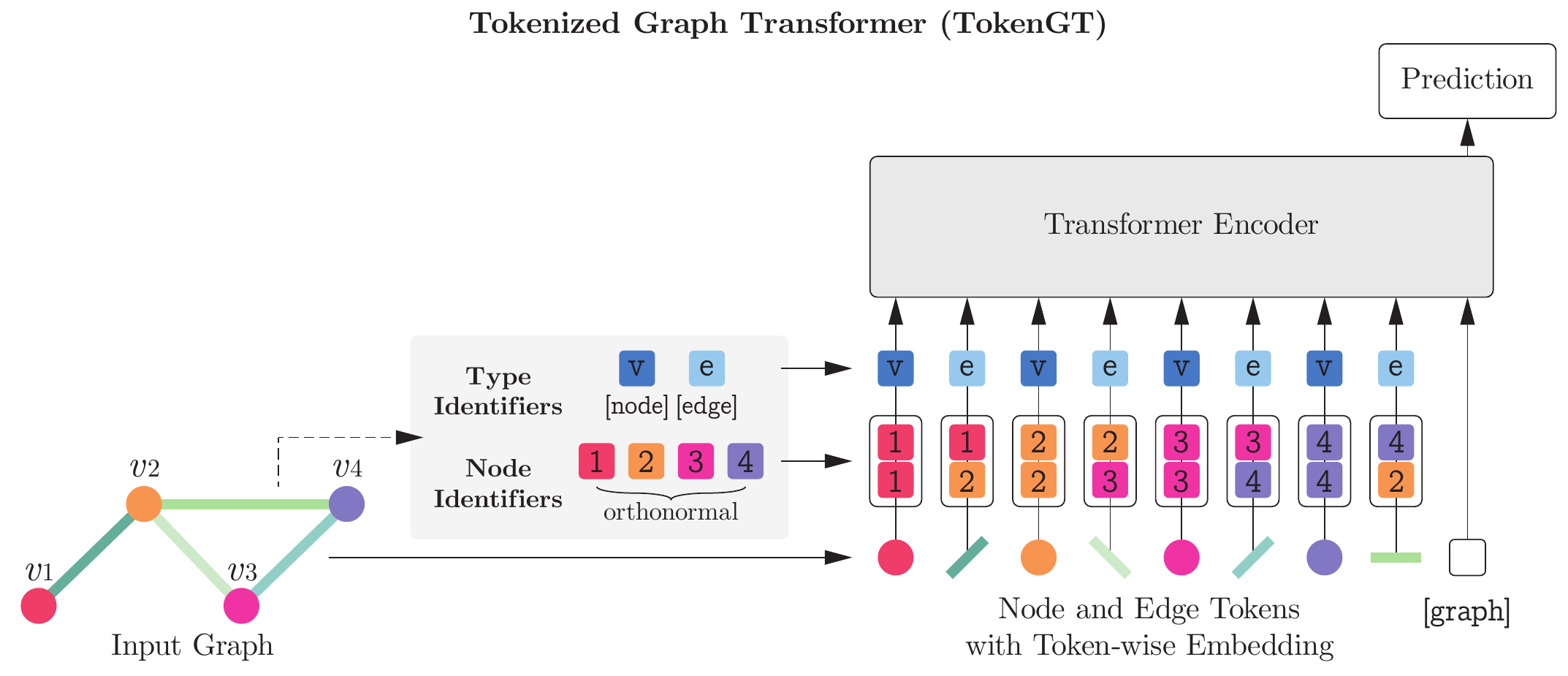}
    \caption{
    Overview of Tokenized Graph Transformer (TokenGT).
    We treat all nodes and edges of an input graph as independent tokens, augment them with orthonormal node identifiers and trainable type identifiers, and feed them to a standard Transformer encoder.
    For graph-level prediction, we follow the common practice~\cite{devlin2019bert, dosovitskiy2021animage} of using an extra trainable \texttt{[graph]} token.
    }
    \label{fig:tokengt}
    \vspace{-0.2em}
\end{figure}

In this section, we present the Tokenized Graph Transformer (TokenGT), a pure Transformer architecture for graphs with token-wise embeddings composed of \emph{node identifiers} and \emph{type identifiers} (Figure~\ref{fig:tokengt}).
Our goal in this section is to provide a practical overview -- for theoretical analysis of the architecture, we guide the readers to Section~\ref{sec:theory}.

Let $\mathcal{G}=(\mathcal{V}, \mathcal{E})$ an input graph with $n$ nodes $\mathcal{V}=\{v_1, ..., v_n\}$ and $m$ edges $\mathcal{E}=\{e_1, ..., e_m\}\subseteq\mathcal{V}^2$, associated with features $\mathbf{X}^\mathcal{V}\in\mathbb{R}^{n\times C}$ and $\mathbf{X}^\mathcal{E}\in\mathbb{R}^{m\times C}$, respectively.
We treat each node and edge as an independent token (thus $(n+m)$ tokens in total) and construct their features by $\mathbf{X}=[\mathbf{X}^\mathcal{V};\mathbf{X}^\mathcal{E}]\in\mathbb{R}^{(n+m)\times C}$.
A na\"ive way to process a graph is to directly provide the tokens $\mathbf{X}$ as input to a Transformer, but it is inappropriate as graph connectivity is discarded.
To thoroughly represent graph structure, we augment the tokens $\mathbf{X}$ with token-wise embeddings, more specifically orthonormal \emph{node identifiers} used for representing the connectivity of the tokens and trainable \emph{type identifiers} that encode whether a token is a node or an edge.
Despite the simplicity, we show that a Transformer applied on these embeddings is a theoretically powerful graph learner.

\paragraph{Node Identifiers}
The first component of token-wise embedding is the orthonormal node identifier that we use to represent the connectivity structure given in the input graph.

For a given input graph $\mathcal{G}=(\mathcal{V}, \mathcal{E})$, we first produce $n$ node-wise orthonormal vectors $\mathbf{P}\in\mathbb{R}^{n\times d_p}$ that we refer to as node identifiers.
Then, we augment the tokens $\mathbf{X}$ with node identifiers as follows.
\begin{itemize}
    \item For each node $v\in\mathcal{V}$, we augment the token $\mathbf{X}_v$ as $[\mathbf{X}_v, \mathbf{P}_v, \mathbf{P}_v]$.
    \item For each edge $(u,v)\in\mathcal{E}$, we augment the token $\mathbf{X}_{(u, v)}$ as $[\mathbf{X}_{(u,v)}, \mathbf{P}_u, \mathbf{P}_v]$.
\end{itemize}
Intuitively, a Transformer operating on the augmented tokens can fully recognize the connectivity structure of the graph since comparing the node identifiers between a pair of tokens reveals their \emph{incidence} information.
For instance, we can tell if an edge~$e=(u, v)$ is connected with a node~$k$ through dot-product (attention) since $[\mathbf{P}_u, \mathbf{P}_v][\mathbf{P}_k, \mathbf{P}_k]^\top=1$ if and only if $k\in(u, v)$ and 0 otherwise. 
This allows the Transformer to identify and exploit the connectivity structure of a graph, for instance by putting more weights on incident pairs when the local operation is important.

Notably, as the node identifiers $\mathbf{P}$ are \emph{only} required to be orthonormal, we have a large degree of freedom in implementation choices.
We outline two practical methods below as examples.
Their implementation details can be found in Appendix~\ref{sec:apdx_node_type_identifiers}.
\begin{itemize}
    \item Orthogonal random features (ORFs), \emph{e.g.}, rows of random orthogonal matrix $\mathbf{Q}\in\mathbb{R}^{n\times n}$ obtained with QR decomposition of random Gaussian matrix $\mathbf{G}\in\mathbb{R}^{n\times n}$~\cite{yu2016orthogonal, choromanski2017the}.
    \item Laplacian eigenvectors obtained from eigendecomposition of graph Laplacian matrix, \emph{i.e.}, rows of $\mathbf{U}$ from $\boldsymbol{\Delta}=\mathbf{I}-\mathbf{D}^{-1/2}\mathbf{A}\mathbf{D}^{-1/2}=\mathbf{U}^\top\boldsymbol{\Lambda}\mathbf{U}$, where $\mathbf{A}\in\mathbb{R}^{n\times n}$ is adjacency matrix, $\mathbf{D}$ is degree matrix, and $\boldsymbol{\Lambda}$, $\mathbf{U}$ correspond to eigenvalues and eigenvectors respectively~\cite{dwivedi2020benchmarking}.
\end{itemize}

Among the two methods, node identifiers generated as ORFs do not encode any information about the graph structure as they are entirely random.
This means the Transformer that operates on the ORF-based node identifiers needs to compile and recognize graph structure only from the incidence information provided by the node identifiers.
Although this is challenging, perhaps surprisingly, we empirically show in Section~\ref{sec:experiments} that Transformers are strong enough to learn meaningful structural representations out of ORF-based node identifiers and outperform GNNs on large-scale task.

In contrast to ORFs, Laplacian eigenvectors provide a kind of graph positional embeddings (graph PEs) that describes the distance between nodes on a graph.
Due to the positional information, it yields better performance compared to ORFs in our experiments in Section~\ref{sec:experiments}.
One interesting aspect of Laplacian eigenvectors is that they can be viewed as a generalization of sinusoidal positional embeddings of NLP Transformers to graphs, as the eigenvectors of 1D chain graphs are sine and cosine functions~\cite{dwivedi2020benchmarking}.
Thus, by choosing Laplacian eigenvectors as node identifiers, our approach can be interpreted as a direct extension of the NLP Transformer for inputs involving relational structures.

\paragraph{Type Identifiers}
The second component of token-wise embedding is the trainable type identifier that encodes whether a token is node or edge.
For a given input graph $\mathcal{G}=(\mathcal{V}, \mathcal{E})$, we first prepare a trainable parameter matrix $\mathbf{E}=[\mathbf{E}^\mathcal{V};\mathbf{E}^\mathcal{E}]\in\mathbb{R}^{2\times d_e}$ that contains two type identifiers $\mathbf{E}^\mathcal{V}$ and $\mathbf{E}^\mathcal{E}$ for nodes and edges respectively.
Then, we further augment the tokens with type identifiers as follows.
\begin{itemize}
    \item For each node $v\in\mathcal{V}$, we augment the token $[\mathbf{X}_v, \mathbf{P}_v, \mathbf{P}_v]$ as $[\mathbf{X}_v, \mathbf{P}_v, \mathbf{P}_v, \mathbf{E}^\mathcal{V}]$.
    \item For each edge $(u,v)\in\mathcal{E}$, we augment the token $[\mathbf{X}_{(u,v)}, \mathbf{P}_u, \mathbf{P}_v]$ as $[\mathbf{X}_{(u,v)}, \mathbf{P}_u, \mathbf{P}_v, \mathbf{E}^\mathcal{E}]$.
\end{itemize}

These embeddings provide information on whether a given token is a node or an edge, which is critical, \emph{e.g.}, when an attention head tries to attend specifically to node tokens and ignore edge tokens.

\paragraph{Main Transformer}
With node identifiers and type identifiers, we obtain augmented token features $\mathbf{X}^{in}\in\mathbb{R}^{(n+m)\times(C+2d_p+d_e)}$, which is further projected by a trainable matrix $w^{in}\in\mathbb{R}^{(C+2d_p+d_e)\times d}$ to be an input to Transformer.
For graph-level prediction, we prepend a special token \texttt{[graph]} with trainable embedding $\mathbf{X}_\texttt{[graph]}\in\mathbb{R}^d$ similar to BERT~\cite{devlin2019bert} and ViT~\cite{dosovitskiy2021animage}.
We utilize the feature of \texttt{[graph]} token at the output of the encoder as the graph representation, on which a linear prediction head is applied to produce the final graph-level prediction.
Overall, the tokens $\mathbf{Z}^{(0)}=[\mathbf{X}_\texttt{[graph]};\mathbf{X}^{in}w^{in}]\in\mathbb{R}^{(1+n+m)\times d}$ are used as the input to the main encoder.
As an encoder, we adopt the standard Transformer~\cite{vaswani2017attention}, which is an alternating stack of multihead self-attention layers (MSA) and feedforward MLP layers.
We provide further details in Appendix~\ref{sec:apdx_proofs_preliminary}.

\paragraph{Inductive Bias}
Similar to Transformers in language and vision~\cite{devlin2019bert, dosovitskiy2021animage}, Tokenized Graph Transformer treats input nodes and edges as independent tokens and applies self-attention to them.
This approach leads to much less inductive bias than current GNNs, where the sparse graph structure, or more fundamentally, \emph{permutation symmetry} of graphs is deliberately baked into each layer~\cite{gilmer2017neural, maron2019invariant, maron2019provably, kim2021transformers}.
For TokenGT, such information is provided entirely as a part of input using token-wise embeddings, and the model has to learn how to interpret and utilize the information from data.
Although such weak inductive bias might raise questions on the expressiveness of the model, our theoretical analysis in Section~\ref{sec:theory} shows that TokenGT is a powerful graph learner thanks to the token-wise embeddings and expressive power of self-attention.
For example, we show that TokenGT is more expressive than all message-passing GNNs under the framework of Gilmer~et~al.~(2017)~\cite{gilmer2017neural}.

\section{Theoretical Analysis}\label{sec:theory}

We now present our theory.
Our key result is that TokenGT, a standard Transformer with node and type identifiers presented in Section~\ref{sec:methods}, is provably \emph{at least} as expressive as the second-order Invariant Graph Network (2-IGN~\cite{maron2019invariant}), which is built upon \emph{all} possible permutation equivariant linear layers on a graph.
This provides solid theoretical guarantees for TokenGT, such as being at least as powerful as the 2-WL graph isomorphism test and more expressive than all message-passing GNNs.
Our theory is based on a general framework on hypergraphs represented as higher-order tensors, which leads to the formulation of order-$k$ TokenGT that is at least as expressive as order-$k$ IGN ($k$-IGN~\cite{maron2019invariant}).

\subsection{Preliminary: Permutation Symmetry and Invariant Graph Networks}\label{sec:preliminary}

\paragraph{Representing and Processing Sets and (Hyper)Graphs}
For a set of $n$ nodes, we often represent their features as $\mathbf{X}\in\mathbb{R}^{n\times d}$ where $\mathbf{X}_i\in\mathbb{R}^d$ is the feature of the $i$-th node.
The set is unordered and, therefore, should be treated invariant to the renumbering of the nodes.
Let $S_n$ the symmetric group or the group of permutations $\pi$ on $[n]=\{1, ..., n\}$.
By $\pi\cdot\mathbf{X}$ we denote permuting rows of $\mathbf{X}$ with $\pi$, \emph{i.e.}, $(\pi\cdot\mathbf{X})_i=\mathbf{X}_{\pi^{-1}(i)}$.
Here, $\mathbf{X}$ and $\pi\cdot\mathbf{X}$ represent the identical set for all $\pi\in S_n$.

Generally, we consider (hyper)graphs represented as order-$k$ tensor $\mathbf{X}\in\mathbb{R}^{n^k\times d}$ with feature $\mathbf{X}_\mathbf{i}=\mathbf{X}_{i_1, ..., i_k}\in\mathbb{R}^d$ attached to (hyper)edge represented as multi-index $\mathbf{i}=(i_1, ..., i_k)\in[n]^k$.
Similar to sets, the tensor should be treated invariant to node renumbering by any $\pi\in S_n$ that acts on $\mathbf{X}$ by $(\pi\cdot\mathbf{X})_{\mathbf{i}}=\mathbf{X}_{\pi^{-1}(\mathbf{i})}$ where $\pi^{-1}(\mathbf{i}) = (\pi^{-1}(i_1), ..., \pi^{-1}(i_k))$.
That is, $\mathbf{X}$ and $\pi\cdot\mathbf{X}$ represent the identical (hyper)graph for all $\pi$.
Due to such symmetry, to build a function $F(\mathbf{X})\approx T$ for tensor $\mathbf{X}$ and target $T$, a suitable way is to make them \emph{invariant} $F(\pi\cdot\mathbf{X})=F(\mathbf{X})$ when the target is a vector or \emph{equivariant} $F(\pi\cdot\mathbf{X})=\pi\cdot F(\mathbf{X})$ when the target is also a tensor, for all $\mathbf{X}\in\mathbb{R}^{n^k\times d}$ and $\pi\in S_n$.

In our theoretical analysis, we work on order-$k$ dense tensor representation $\mathbf{X}\in\mathbb{R}^{n^k\times d}$ of a graph as they can represent node features ($k=1$), edge features ($k=2$), or hyperedge features ($k>2$) in a unified manner.
This is interchangeable but slightly different from the sparse representation of a graph with edge set $\mathcal{E}$ used in Section~\ref{sec:methods}.
Nevertheless, in Section~\ref{sec:experiments} we empirically verify that our key theoretical findings work equally well for dense and sparse graphs.

\paragraph{Invariant Graph Network}

We mainly develop our theoretical analysis upon \emph{Invariant Graph Networks (IGNs)}~\cite{maron2019invariant, maron2019provably}, a family of expressive graph networks derived from the permutation symmetry of tensor representation of graphs.
Here we provide a summary.
In general, we define:
\begin{defn}\label{defn:invariant_graph_network}
An order-$k$ Invariant Graph Network ($k$-IGN) is a function $F_k:\mathbb{R}^{n^k\times d_0}\to\mathbb{R}$ written as the following:
\begin{align}
    F_k = \textnormal{MLP}\circ L_{k\to 0}\circ L_{k\to k}^{(T)}\circ\sigma\circ...\circ\sigma\circ L_{k\to k}^{(1)},\label{eqn:invariant_graph_network}
\end{align}
where each $L_{k\to k}^{(t)}$ is equivariant linear layer~\cite{maron2019invariant} from $\mathbb{R}^{n^k\times d_{t-1}}$ to $\mathbb{R}^{n^k\times d_t}$, $\sigma$ is activation function, and $L_{k\to 0}$ is a invariant linear layer from $\mathbb{R}^{n^k\times d_T}$ to $\mathbb{R}$.
\end{defn}

A body of previous work have shown appealing theoretical properties of $k$-IGN, including universal approximation~\cite{maron2019on} and alignment to $k$-Weisfeiler-Lehman ($k$-WL) graph isomorphism test~\cite{maron2019provably, chen2020can}.
In particular, it is known that $k$-IGNs are theoretically at least as powerful as the $k$-WL test~\cite{maron2019provably}.
It is also known that 2-IGNs are already more expressive~\cite{maron2019invariant, kim2021transformers} than all message-passing GNNs under the framework of Gilmer~et~al.~(2017)~\cite{gilmer2017neural}.

The core building block of IGN is invariant and equivariant \emph{linear} layers~\cite{maron2019invariant} with maximal expressiveness while respecting node permutation symmetry.
The layers are defined as follows:
\begin{defn}\label{defn:equivariant_linear_layer}
An equivariant linear layer is a function $L_{k\rightarrow l}:\mathbb{R}^{n^k\times d}\to\mathbb{R}^{n^l\times d'}$ written as follows for order-$k$ input $\mathbf{X}\in\mathbb{R}^{n^k\times d}$:
\begin{align}
    L_{k\rightarrow l}(\mathbf{X})_{\mathbf{i}} = \sum_{\mu}{\sum_{\mathbf{j}}{\mathbf{B}^{\mu}_{\mathbf{i}, \mathbf{j}}\mathbf{X}_{\mathbf{j}}w_{\mu}}} + \sum_{\lambda}{\mathbf{C}^{\lambda}_{\mathbf{i}}b_{\lambda}},\label{eqn:equivariant_linear_layer}
\end{align}
where $\mathbf{i}\in[n]^l,\mathbf{j}\in[n]^k$ are multi-indices, $w_\mu\in\mathbb{R}^{d\times d'}$, $b_\lambda\in\mathbb{R}^{d'}$ are weight and bias parameters, and $\mathbf{B}^\mu\in\mathbb{R}^{n^{l+k}}$ and $\mathbf{C}^\lambda\in\mathbb{R}^{n^l}$ are binary basis tensors corresponding to order-$(l+k)$ and order-$l$ equivalence classes $\mu$ and $\lambda$, respectively.
Invariant linear layer is a special case of $L_{k\rightarrow l}$ with $l=0$.
\end{defn}

We provide the definition of the equivalence classes and basis tensors in Appendix~\ref{sec:apdx_proofs_preliminary}.
For now, it is sufficient to know that the basis tensors are binary tensors that form the orthogonal basis of the full space of linear equivariant layers.
In general, in Eq.~\eqref{eqn:equivariant_linear_layer} it is known that there exists $\textnormal{bell}(k+l)$ number of basis tensors $\mathbf{B}^{\mu}$ for the weight and $\textnormal{bell}(l)$ number of basis tensors $\mathbf{C}^{\lambda}$ for the bias.

\subsection{Can Self-Attention Approximate Equivariant Basis?}\label{sec:intuition}

Now, we present an intuition that connects Transformer (Section~\ref{sec:methods}) and equivariant linear layer~(Definition~\ref{defn:equivariant_linear_layer}).
For that, we write out the multihead self-attention layer as follows:
\begin{align}
    \textnormal{MSA}(\mathbf{X})_i = \sum_{h=1}^H{\sum_j\boldsymbol{\alpha}_{ij}^h\mathbf{X}_j w_h^V w_h^O}\textnormal{ where }\boldsymbol{\alpha}^h = \textnormal{softmax}\left(\frac{\mathbf{X}w_h^Q(\mathbf{X}w_h^K)^\top}{\sqrt{d_H}}\right),\label{eqn:transformer_multihead_self_attention_recap}
\end{align}
where $H$ is number of heads, $d_H$ is head size, and $w_h^Q,w_h^K\in\mathbb{R}^{d\times d_H}$, $w_h^V\in\mathbb{R}^{d\times d_v}$ $w_h^O\in\mathbb{R}^{d_v\times d}$.

Our intuition is that the weighted sum of values with self-attention matrix $\boldsymbol{\alpha}^h$ in Eq.~\eqref{eqn:transformer_multihead_self_attention_recap} is analogous to the masked sum with basis tensor $\mathbf{B}^\mu$ in Eq.~\eqref{eqn:equivariant_linear_layer} up to normalization.
This naturally leads to the following question: for a given equivariant layer $L_{k\to k}:\mathbb{R}^{n^k\times d}\to\mathbb{R}^{n^k\times d}$, can we use a Transformer layer with multihead self-attention $\textnormal{MSA}:\mathbb{R}^{N\times d'}\to\mathbb{R}^{N\times d'}$ with $N=n^k$ to accurately approximate $L_{k\to k}$ by having $H = \textnormal{bell}(2k)$ attention heads approximate each equivariant basis $\mathbf{B}^\mu$?

We show that this can be possible, but only if we provide appropriate auxiliary information to input.
For example, let us consider first-order layer $L_{1\rightarrow 1}$.
The layer has $\text{bell}(2)=2$ basis tensors $\mathbf{B}^{\mu_1}=\mathbf{I}$ and $\mathbf{B}^{\mu_2}=\mathbf{11}^\top-\mathbf{I}$ for the weight, and $\text{bell}(1)=1$ basis tensor $\mathbf{C}^{\lambda_1}=\mathbf{1}$ for the bias.
Given an input set $\mathbf{X}\in\mathbb{R}^{n\times d}$ it computes the following with $w_1,w_2\in\mathbb{R}^{d\times d}$, $b\in\mathbb{R}^d$:
\begin{align}
    L_{1\rightarrow 1}(\mathbf{X}) = \mathbf{I}\mathbf{X}w_1+(\mathbf{11}^\top-\mathbf{I})\mathbf{X}w_2+\mathbf{1}b^{\top}.\label{eqn:equivariant_linear_layer_first_order}
\end{align}
Now consider approximating basis tensor $\mathbf{B}^{\mu_1}=\mathbf{I}$ with an attention matrix $\boldsymbol{\alpha}^1$.
The approximation is accurate when $i$-th query always only attends to $i$-th key and ignores the rest.
To achieve the attention structure consistently, \emph{i.e.}, agnostic to input $\mathbf{X}$, we need to provide auxiliary input that self-attention can "latch onto" to faithfully approximate $\boldsymbol{\alpha}^1\approx\mathbf{I}$.
Without this, attention must entirely rely on the inputs $\mathbf{X}$, which is unreliable and can lead to approximation failure, \emph{e.g.}, when $\mathbf{X}$ has repeated rows.

For the auxiliary information, we prepare $n$ node-wise orthonormal vectors $\mathbf{P}\in\mathbb{R}^{n\times d_p}$ (note that this is identical to node identifiers in Section~\ref{sec:methods}), and augment the input to $\mathbf{X}^{in}=[\mathbf{X}, \mathbf{P}]\in\mathbb{R}^{n\times(d+d_p)}$.
Let us assume that the query and key projections in Eq.~\eqref{eqn:transformer_multihead_self_attention_recap} ignore $\mathbf{X}$ and only leave $\mathbf{P}$ scaled by $\sqrt{a}$ with $a>0$.
Then attention matrix is computed as $\boldsymbol{\alpha}^1 = \textnormal{softmax}(\mathbf{S})\textnormal{ where }\mathbf{S}_{ij}=a\mathbf{P}_i^\top\mathbf{P}_j$.
Here, due to the orthonormality of $\mathbf{P}$, we have $\mathbf{P}_i^\top\mathbf{P}_j=1$ only if $i=j$ and otherwise 0, which leads to $\mathbf{S}=a\mathbf{I}$.
With $a\to\infty$ by scaling up the query and key projection weights, the softmax becomes arbitrarily close to the hardmax operator, and we obtain the following:
\begin{align}
    \boldsymbol{\alpha}^1 = \textnormal{softmax}(a\mathbf{I})\to \mathbf{I}\textnormal{ as }a\to\infty.
\end{align}
Thus, self-attention can utilize the auxiliary information $\mathbf{P}$ to achieve an input-agnostic approximation of $\boldsymbol{\alpha}^1$ to $\mathbf{I}$.
Notably, we can achieve a similar approximation for $\mathbf{B}^{\mu_2}=\mathbf{11}^\top-\mathbf{I}$ using the \emph{same} $\mathbf{P}$ by flipping the sign of keys, which gives $\boldsymbol{\alpha}^2=\textnormal{softmax}(-a\mathbf{I})$ due to orthonormality.
By sending $a\to\infty$, now attention from the $i$-th query to the $i$-th key is suppressed, and we obtain the following:
\begin{align}
    \boldsymbol{\alpha}^2=\textnormal{softmax}\left(-a\mathbf{I}\right)\to \frac{1}{n-1}(\mathbf{11}^\top-\mathbf{I})\textnormal{ as }a\to\infty.
\end{align}
Note that this approximation is accurate only up to row normalization as rows of $\boldsymbol{\alpha}^2$ always sum to one due to softmax, while $\mathbf{B}^{\mu_2}=\mathbf{11}^\top-\mathbf{I}$ is binary.
In our proofs of the theoretical results, we perform appropriate denormalization with MLP after MSA to achieve an accurate approximation.

Overall, we see that simple auxiliary input $\mathbf{P}$ suffices for two attention heads to approximate the equivariant basis of $L_{1\to 1}$ accurately.
We now question the following.
Given appropriate auxiliary information as input, can a Transformer layer with $\textnormal{bell}(2k)$ attention heads accurately approximate $L_{k\to k}$ by having each head approximate each equivariant basis $\mathbf{B}^\mu$?
What would be the sufficient auxiliary input?
We answer the question by showing that, with (order-$k$ generalized) node and type identifiers presented in Section~\ref{sec:methods}, Transformer layers can accurately approximate equivariant layers $L_{k\to k}$ via input-agnostic head-wise approximation of each equivariant basis.

\subsection{Pure Transformers are Powerful Graph Learners}\label{sec:theoretical_results}

We now present our main theoretical results that extend the discussions in Section~\ref{sec:intuition} to any order~$k$.
Note that $k=2$ corresponds to TokenGT for graphs presented in Section~\ref{sec:methods}.
With $k>2$, we naturally extend TokenGT to hypergraphs.
All proofs can be found in Appendix~\ref{sec:apdx_proofs}.

We first introduce generalized node and type identifiers (Section~\ref{sec:methods}) for order-$k$ tensors $\mathbf{X}\in\mathbb{R}^{n^k\times d}$.
We define the node identifier $\mathbf{P}\in\mathbb{R}^{n\times d_p}$ as an orthonormal matrix with $n$ rows, and the type identifier as a trainable matrix $\mathbf{E}\in\mathbb{R}^{\textnormal{bell}(k)\times d_e}$ that contains $\textnormal{bell}(k)$ rows $\mathbf{E}^{\gamma_1}, ..., \mathbf{E}^{\gamma_{\textnormal{bell}(k)}}$, each of which is designated for an order-$k$ equivalence class $\gamma$.
Then, we augment each entry of input tensor as $[\mathbf{X}_{i_1, ..., i_k}, \mathbf{P}_{i_1}, ...,\mathbf{P}_{i_k}, \mathbf{E}^\gamma]$
where $(i_1, ..., i_k)\in\gamma$. %where $\gamma$ is the equivalence class of multi-index $(i_1, ..., i_k)$.

Let us exemplify.
For $k=1$~(sets), each $i$-th entry is augmented as $[\mathbf{X}_i, \mathbf{P}_i, \mathbf{E}^{\gamma_1}]$, consistent with our discussion in Section~\ref{sec:intuition}.
For $k=2$~(graphs), each $(i, i)$-th entry is augmented as $[\mathbf{X}_{ii}, \mathbf{P}_i, \mathbf{P}_i, \mathbf{E}^{\gamma_1}]$ and each $(i, j)$-th entry ($i\neq j$) is augmented as $[\mathbf{X}_{ij}, \mathbf{P}_i, \mathbf{P}_j, \mathbf{E}^{\gamma_2}]$.
This is consistent with TokenGT in Section~\ref{sec:methods}, which augments nodes with $\mathbf{E}^{\mathcal{V}}=\mathbf{E}^{\gamma_1}$ and edges with $\mathbf{E}^{\mathcal{E}}=\mathbf{E}^{\gamma_2}$.

With node and type identifiers, we obtain augmented order-$k$ tensor $\mathbf{X}^{in}\in\mathbb{R}^{n^k\times(d+kd_p+d_e)}$.
We use a trainable projection $w^{in}\in\mathbb{R}^{(d+kd_p+d_e)\times d_\mathcal{T}}$ to map them to hidden dimension $d_\mathcal{T}$ of a Transformer.
We now show that self-attention on $\mathbf{X}^{in}w^{in}$ can accurately approximate equivariant basis:
\begin{lemma}\label{lemma:approximation_equivariant_basis}
For all $\mathbf{X}\in\mathbb{R}^{n^k\times d}$ and their augmentation $\mathbf{X}^{in}$, self-attention coefficients $\boldsymbol{\alpha}^h$ (Eq.~\eqref{eqn:transformer_multihead_self_attention_recap}) computed with $\mathbf{X}^{in}w^{in}$ can approximate any basis tensor $\mathbf{B}^{\mu}\in\mathbb{R}^{n^{2k}}$ of order-$k$ equivariant linear layer $L_{k\to k}$ (Definition~\ref{defn:equivariant_linear_layer}) to arbitrary precision up to normalization.
\end{lemma}
Consequently, with the node and type identifiers, a collection of $\textnormal{bell}(2k)$ attention heads can approximate the collection of all basis tensors of order-$k$ equivariant layer.
This leads to the following:
\begin{thm}\label{thm:approximation_equivariant_layer}
For all $\mathbf{X}\in\mathbb{R}^{n^k\times d}$ and their augmentation $\mathbf{X}^{in}$, a Transformer layer with $\textnormal{bell}(2k)$ self-attention heads that operates on $\mathbf{X}^{in}w^{in}$ can approximate an order-$k$ equivariant linear layer $L_{k\to k}(\mathbf{X})$ (Definition~\ref{defn:equivariant_linear_layer}) to arbitrary precision.
\end{thm}
While the approximation in Lemma~\ref{lemma:approximation_equivariant_basis} is only accurate up to normalization over inputs (keys) due to softmax normalization, for the approximation in Theorem~\ref{thm:approximation_equivariant_layer} we perform appropriate denormalization using MLP after multihead self-attention and can obtain an accurate approximation.

By extending the result to multiple layers, we arrive at the following:
\begin{thm}\label{thm:approximation_order_k_ign}
For all $\mathbf{X}\in\mathbb{R}^{n^k\times d}$ and their augmentation $\mathbf{X}^{in}$, a Transformer composed of $T$ layers that operates on $\mathbf{X}^{in}w^{in}$ followed by sum-pooling and MLP can approximate an $k$-IGN $F_k(\mathbf{X})$ (Definition~\ref{defn:invariant_graph_network}) to arbitrary precision.
\end{thm}
This directly leads to the following corollary:
\begin{corollary}\label{corollary:expressiveness_ign}
A Transformer on node and type identifiers in Theorem~\ref{thm:approximation_order_k_ign} is at least as expressive as $k$-IGN composed of order-$k$ equivariant linear layers.
\end{corollary}
Corollary~\ref{corollary:expressiveness_ign} allows us to draw previous theoretical results on the expressiveness of $k$-IGN~\cite{maron2019provably, maron2019invariant, kim2021transformers} and use them to lower-bound the provable expressiveness of a standard Transformer:
\begin{corollary}\label{corollary:expressiveness_wl_gnn}
A Transformer on node and type identifiers in Theorem~\ref{thm:approximation_order_k_ign} is at least as powerful as $k$-WL graph isomorphism test and is more expressive than all message-passing GNNs within the framework of Gilmer~et~al.~(2017)~\cite{gilmer2017neural}.
\end{corollary}

\section{Related Work}\label{sec:related_work}
\vspace{-0.1cm}
We outline relevant work including equivariant neural networks, theory on expressive power of Transformers and their connection to modeling equivariance, and Transformers for graphs.

\vspace{-0.1cm}
\paragraph{Equivariant Neural Networks}
A machine learning task is often invariant or equivariant to specific symmetry of input data, \emph{e.g.}, image classification is invariant to the translation of an input image.
A large body of literature advocated baking the invariance or equivariance into a neural network as a type of inductive bias (\emph{e.g.}, translation equivariance of image convolution), showing that it reduces the number of parameters and improves generalization for a wide range of learning tasks involving various geometric structures~\cite{cohen2016group, cohen2017steerable, weiler20183d, thomas2018tensor, maron2020onlearning, pan2022permutation, serviansky2020set, bronstein2017geometric, kim2021transformers, lee2019set}.
Ravanbakhsh~et~al.~(2017)~\cite{ravanbakhsh2017equivariance} showed that any equivariant layer for discrete group actions is equivalent to a specific parameter sharing structure.
Zaheer~et~al.~(2017)~\cite{zaheer2017deep} and Maron~et~al.~(2019)~\cite{maron2019invariant} derived the parameter sharing for node permutation-symmetric data (sets and (hyper)graphs), which gives the maximally expressive equivariant linear layers and $k$-IGN in Section~\ref{sec:preliminary}.
The work on equivariant neural networks underlie our theory of how a standard Transformer can be a powerful learner for sets and (hyper)graphs.
%Recently, Wang~et~al.~(2022)~\cite{wang2022approximately} explored flexible equivariant networks with relaxed (\emph{soft}) equivariance constraints, and Yeh~et~al.~(2022)~\cite{yeh2022equivariance} studied discovering equivariance from data by optimizing parameter sharing structure.

\vspace{-0.1cm}
\paragraph{Expressive Power of Transformers and Its Connection to Equivariance}
%In contrast to the equivariant neural networks, 
Recent work involving Transformers often focus on minimizing the domain- and task-specific inductive bias and scaling the model and data so that any useful computation structure can be learned~\cite{dosovitskiy2021animage, andrew2021perceiver, andrew2021perceiverio, brown2020language, devlin2019bert, chen2021decision, lee2019set}.
The success of this approach is, to some degree, attributed to the high expressive power of Transformers that allows learning diverse functions suited for the data at hand~\cite{yun2020are, lee2019set, bhattamishra2020on, bhojanapalli2020low, likhosherstov2021on}.
Recent theory has shown that Transformers are expressive enough to even model certain equivariant functions~\cite{andreoli2019convolution, cordonnier2020on, lee2019set}.
Andreoli~et~al.~(2019)~\cite{andreoli2019convolution} cast self-attention and convolution into a unified framework using basis tensors similar to ones in Section~\ref{sec:preliminary}.
Cordonnier~et~al.~(2020)~\cite{cordonnier2020on} advanced the idea and showed that Transformers with relative positional encodings can approximate any image convolution layers.
%if the number of heads can grow quadratically to kernel width.
Lee~et~al.~(2019)~\cite{lee2019set} and Kim~et~al.~(2021)~\cite{kim2021transformers} showed that Transformers can model equivariant linear layers for sets~\cite{zaheer2017deep}, which can be viewed as the first-order case of our theory (see Section~\ref{sec:intuition}).
To our knowledge, our work is the first to show that standard Transformers are expressive enough to provably model maximally expressive equivariant layers and $k$-IGN for (hyper)graphs with $k\geq 2$.
%Furthermore, our approximation uses a fixed number of heads (15 for graphs) and does not require relative positional encodings.

\vspace{-0.1cm}
\paragraph{Transformers for Graphs}
Unlike in language and vision, developing Transformers for graphs is challenging due to \textbf{(1)} the presence of edge connectivity and \textbf{(2)} the absence of canonical node ordering that prevents adopting simple positional encodings~\cite{min2022transformer}.
To incorporate the connectivity of edges, early methods restricted self-attention to local neighborhoods (thus reducing to message-passing)~\cite{dwivedi2020a, nguyen2022universal, velikovic2018graph} or used global self-attention with auxiliary message-passing modules~\cite{rong2020self, lin2021mesh}.
As message-passing suffers from limited expressive power~\cite{xu2019how} and oversmoothing~\cite{li2018deeper, cai2020a, oono2020graph}, recent works often discard them and use global self-attention on nodes with heuristic modifications to process edges~\cite{ying2021do, hussain2021edge, park2022grpe, kreuzer2021rethinking, lim2022sign}.
Ying~et~al.~(2021)~\cite{ying2021do} proposed to inject edge encoding based on shortest paths through self-attention bias.
Kreuzer~et~al.~(2021)~\cite{kreuzer2021rethinking} proposed to incorporate edges into self-attention matrix via elementwise multiplication.
On the contrary, we leave the self-attention unmodified and provide both nodes and edges with certain token-wise embeddings (Section~\ref{sec:methods}) as its input.
%To our knowledge, our approach gives the first unmodified standard Transformer that works well for graph learning.
To incorporate graph structure into nodes, on the other hand, some approaches focus on developing graph positional encoding, \emph{e.g.}, based on Laplacian eigenvectors~\cite{dwivedi2020benchmarking, lim2022sign, kreuzer2021rethinking}.
While these can be directly incorporated into our work via auxiliary node identifiers for better performance, we leave this as future work.
We further note that current graph Transformers that utilize Laplacian positional encoding rely heavily on heuristic edge encoding~\cite{hussain2021edge, kreuzer2021rethinking} while ours does not.
Another closely related approach is the Higher-order Transformer~\cite{kim2021transformers} which generalizes $k$-IGN with masked self-attention.
While it is highly complex to implement due to hard-coded head-wise equivariant masks, our method can be implemented effortlessly using any available implementation of standard Transformer.
Furthermore, our method is more flexible as the model can choose to use different attention heads to focus on a specific equivariant operator (\emph{e.g.}, local propagation) if needed.
We further discuss the difficulty in applying linear attention to graph Transformers in Appendix~\ref{sec:apdx_extended_related_work}.

\cutparagraphup
\section{Experiments}\label{sec:experiments}
\vspace{-0.1cm}

We first conduct a synthetic experiment that directly confirms our key claims in Lemma~\ref{lemma:approximation_equivariant_basis}~(Section~\ref{sec:theory}).
Then, we empirically explore the capability of Tokenized Graph Transformer~(TokenGT)~(Section~\ref{sec:methods}) using the PCQM4Mv2 large-scale quantum chemistry regression dataset~\cite{hu2021ogb}.
We further present experiments on transductive node classification datasets involving large graphs in Appendix~\ref{sec:apdx_additional_results_transductive_node_classification}.

\vspace{-0.1cm}
\subsection{Approximating Second-Order Equivariant Basis}\label{sec:experiment_synthetic}
\begin{table}[!t]
\caption{
Second-order equivariant basis approximation.
We report average and standard deviation of L2 error averaged over heads over 3 runs.
For Random/ORF~(first-order), we sample random embeddings independently for each token.
}
\centering
\begin{adjustbox}{width=0.9\textwidth}
    \begin{tabular}{lc|cc|cc}\label{table:equivariant_basis_approximation_std}
        \\\Xhline{2\arrayrulewidth}\\[-1em]
         & & \multicolumn{2}{c|}{dense input} & \multicolumn{2}{c}{sparse input} \\
        node id. & type id. & train L2 $\downarrow$ & test L2 $\downarrow$ & train L2 $\downarrow$ & test L2 $\downarrow$ \\
        \Xhline{2\arrayrulewidth}\\[-1em]
        $\times$ & $\times$ & 47.95 $\pm$ 0.600 & 53.93 $\pm$ 1.426 & 29.88 $\pm$ 0.450 & 34.70 $\pm$ 1.167 \\
        \Xhline{2\arrayrulewidth}\\[-1em]
        $\times$ & $\bigcirc$ & 32.38 $\pm$ 0.448 & 40.06 $\pm$ 1.202 & 15.92 $\pm$ 0.275 & 20.39 $\pm$ 0.765 \\
        Random (first-order) & $\bigcirc$ & 32.19 $\pm$ 0.476 & 32.49 $\pm$ 3.687 & 15.87 $\pm$ 0.247 & 16.56 $\pm$ 0.904 \\
        ORF (first-order) & $\bigcirc$ & 32.35 $\pm$ 0.369  & 39.87 $\pm$ 1.263 & 15.87 $\pm$ 0.247 & 16.56 $\pm$ 0.908  \\
        \Xhline{2\arrayrulewidth}\\[-1em]
        Random & $\times$ & 5.909 $\pm$ 0.019  & 5.548 $\pm$ 0.090 & 8.152 $\pm$ 0.042 & 8.270 $\pm$ 0.285 \\
        ORF & $\times$ & 5.472 $\pm$ 0.035 & 5.143 $\pm$ 0.078 & 7.167 $\pm$ 0.025 & 7.190 $\pm$ 0.217 \\
        Laplacian eigenvector & $\times$ & 1.899 $\pm$ 3.050 & 1.702 $\pm$ 2.912 & 0.288 $\pm$ 0.019 & 0.064 $\pm$ 0.010 \\
        \Xhline{2\arrayrulewidth}\\[-1em]
        Random & $\bigcirc$ & 0.375 $\pm$ 0.009 & 0.234 $\pm$ 0.011 & 0.990 $\pm$ 0.108 & 0.875 $\pm$ 0.042 \\
        ORF & $\bigcirc$ & 0.080 $\pm$ 0.001 & 0.009 $\pm$ 5e-5 & 0.129 $\pm$ 0.002 & \textbf{0.011 $\pm$ 0.002} \\
        Laplacian eigenvector & $\bigcirc$ & \textbf{0.053 $\pm$ 1.5e-5} & \textbf{0.005 $\pm$ 1e-4} & \textbf{0.101 $\pm$ 0.003} & 0.019 $\pm$ 0.007 \\
        \Xhline{2\arrayrulewidth}
    \end{tabular}
\end{adjustbox}
%\vspace{-0.1cm}
\end{table}
\begin{figure}[!t]
    \vspace{-0.2cm}
    \centering
    \includegraphics[width=0.99\textwidth]{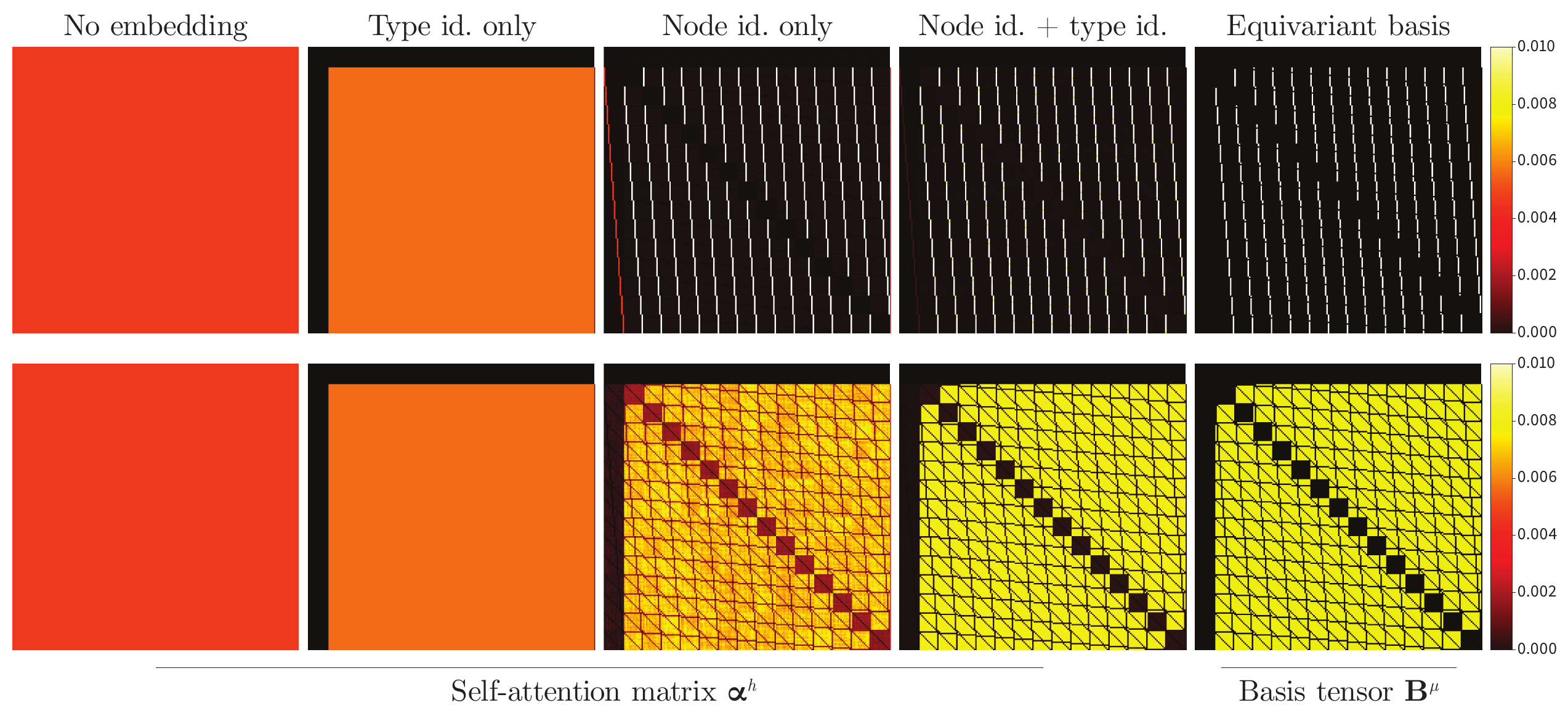}
    \vspace{-0.2cm}
    \caption{
    Self-attention maps learned under various node and type identifier configurations for two target equivariant basis tensors (out of 15).
    For better visualization, we clamp the entries by 0.01.
    Self-attention learns acute patterns coherent to equivariant basis when orthonormal node identifiers and type identifiers are both provided as input.
    More images can be found in Appendix~\ref{sec:apdx_additional_results_synthetic}.
    }
    \label{fig:self_attention_maps}
    \vspace{-0.5cm}
\end{figure}

As in Theorem~\ref{thm:approximation_equivariant_layer} and \ref{thm:approximation_order_k_ign}~(Section~\ref{sec:theory}), our argument on the expressive power of TokenGT relies on its capability to approximate order-$k$ permutation equivariant linear layers $L_{k\to k}$ (Definition~\ref{defn:equivariant_linear_layer}).
Specifically, Lemma~\ref{lemma:approximation_equivariant_basis} states that such capability depends on the ability of each self-attention head $\boldsymbol{\alpha}^1, ..., \boldsymbol{\alpha}^H$~(Eq.~\eqref{eqn:transformer_multihead_self_attention_recap}) to accurately approximate each equivariant basis $\mathbf{B}^{\mu_1}, ..., \mathbf{B}^{\mu_{\textnormal{bell}(2k)}}$~(Definition~\ref{defn:equivariant_linear_layer}) up to normalization.

We verify this claim for $k=2$ (second-order; graphs) in a synthetic setup using Barab\'asi-Albert random graphs.
We use a multihead self-attention layer (Eq.~\eqref{eqn:transformer_multihead_self_attention_recap}) with $\textnormal{bell}(2+2)=15$ heads and explicitly supervise head-wise attention scores $\boldsymbol{\alpha}^h$ to approximate each (normalized) equivariant basis tensor $\mathbf{B}^{\mu_h}$ by minimizing L2 loss.
Having the layer hyperparameters fixed, we provide different combinations of node and type identifiers, and test if multihead self-attention can jointly approximate \emph{all} 15 equivariant basis on unseen graphs.
We experiment with both dense and sparse graph representations; for graphs with $n$ nodes and $m$ edges, the dense graph considers all $n^2$ pair-wise edges as input as in Section~\ref{sec:theory}, whereas the sparse graph considers only the present $m$ edges as in Section~\ref{sec:methods}.
Further details can be found in Appendix~\ref{sec:apdx_experiment_details_synthetic}.

We outline the results in Table~\ref{table:equivariant_basis_approximation_std}.
Consistent with Lemma~\ref{lemma:approximation_equivariant_basis}, self-attention achieves accurate approximation of equivariant basis only when both the orthonormal node identifiers and type identifiers are given.
Here, Laplacian eigenvectors (Lap,~$\bigcirc$) often yield slightly better results than orthogonal random features (ORF,~$\bigcirc$) presumably due to less stochasticity.
Interestingly, we see that self-attention transfers the learned (pseudo-)equivariant self-attention structure to unseen graphs near perfectly.
Non-orthogonal random embeddings lead to inaccurate approximation (Random,~$\bigcirc$), highlighting the importance of orthogonality of node identifiers.
The approximation is also inaccurate when we sample ORF $\mathbf{P}_t$ independently for each token $t$ (ORF~(first-order),~$\bigcirc$) instead of using concatenated node identifiers $[\mathbf{P}_u, \mathbf{P}_v]$ for token $(u, v)$.
This supports our argument in Section~\ref{sec:methods} that the incidence information implicitly provided via node identifiers plays a key role in approximation.

In Figure~\ref{fig:self_attention_maps}, we provide a visualization of self-attention maps learned under various node and type identifier choices.
Additional results can be found in Appendix~\ref{sec:apdx_additional_results_synthetic}.

\vspace{-0.15cm}
\subsection{Large-Scale Graph Learning}\label{sec:experiment_pcqm4mv2}
\vspace{-0.15cm}
\begin{table}[!t]
\vspace{-0.2cm}
\caption{
Results on PCQM4Mv2 large-scale graph regression benchmark.
We report the Mean Absolute Error (MAE) on the validation set, and report MAE on the unavailable test set if possible.
}
\vspace{-0.2cm}
\centering
\begin{adjustbox}{width=0.75\textwidth}
\begin{tabular}{lcccl}\label{table:pcqm4mv2}
    \\\Xhline{2\arrayrulewidth}\\[-1em]
    method & \# parameters & valid MAE $\downarrow$ & test-dev MAE $\downarrow$ & asymptotics \\
    \Xhline{2\arrayrulewidth}\\[-1em]
    \multicolumn{4}{l}{\emph{Message-passing GNNs}} \\
    \hline\\[-1em]
    GCN~\cite{hu2021ogb} & 2.0M & 0.1379 & 0.1398 & $\mathcal{O}(n + m)$ \\
    GIN~\cite{hu2021ogb} & 3.8M & 0.1195 & 0.1218 & $\mathcal{O}(n + m)$ \\
    GAT & 6.7M & 0.1302 & N/A & $\mathcal{O}(n + m)$ \\
    GCN-VN~\cite{hu2021ogb} & 4.9M & 0.1153 & 0.1152 & $\mathcal{O}(n + m)$ \\
    GIN-VN~\cite{hu2021ogb} & 6.7M & 0.1083 & 0.1084 & $\mathcal{O}(n + m)$ \\
    GAT-VN & 6.7M & 0.1192 & N/A & $\mathcal{O}(n + m)$ \\
    GAT-VN (large) & 55.2M & 0.1361 & N/A & $\mathcal{O}(n + m)$ \\
    \Xhline{2\arrayrulewidth}\\[-1em]
    \multicolumn{4}{l}{\emph{Transformers with strong graph-specific modifications}} \\
    \hline\\[-1em]
    Graphormer~\cite{shi2022benchmarking} & 48.3M & 0.0864 & N/A & $\mathcal{O}(n^2)$ \\
    EGT~\cite{hussain2021edge} & 89.3M & 0.0869 & 0.0872 & $\mathcal{O}(n^2)$ \\
    GRPE~\cite{park2022grpe} & 46.2M & 0.0890 & 0.0898 & $\mathcal{O}(n^2)$ \\
    \Xhline{2\arrayrulewidth}\\[-1em]
    \multicolumn{4}{l}{\emph{Pure Transformers}} \\
    \hline\\[-1em]
    Transformer & 48.5M & 0.2340 & N/A & $\mathcal{O}((n+m)^2)$ \\
    TokenGT~(ORF) & 48.6M & 0.0962 & N/A & $\mathcal{O}((n+m)^2)$ \\
    TokenGT~(Lap) & 48.5M & 0.0910 & 0.0919 & $\mathcal{O}((n+m)^2)$ \\
    TokenGT~(Lap) + Performer & 48.5M & 0.0935 & N/A & $\mathcal{O}(n+m)$ \\
    \Xhline{2\arrayrulewidth}
\end{tabular}
\end{adjustbox}
\vspace{-0.35cm}
\end{table}
\begin{figure}[!t]
    \centering
    \includegraphics[width=0.73\textwidth]{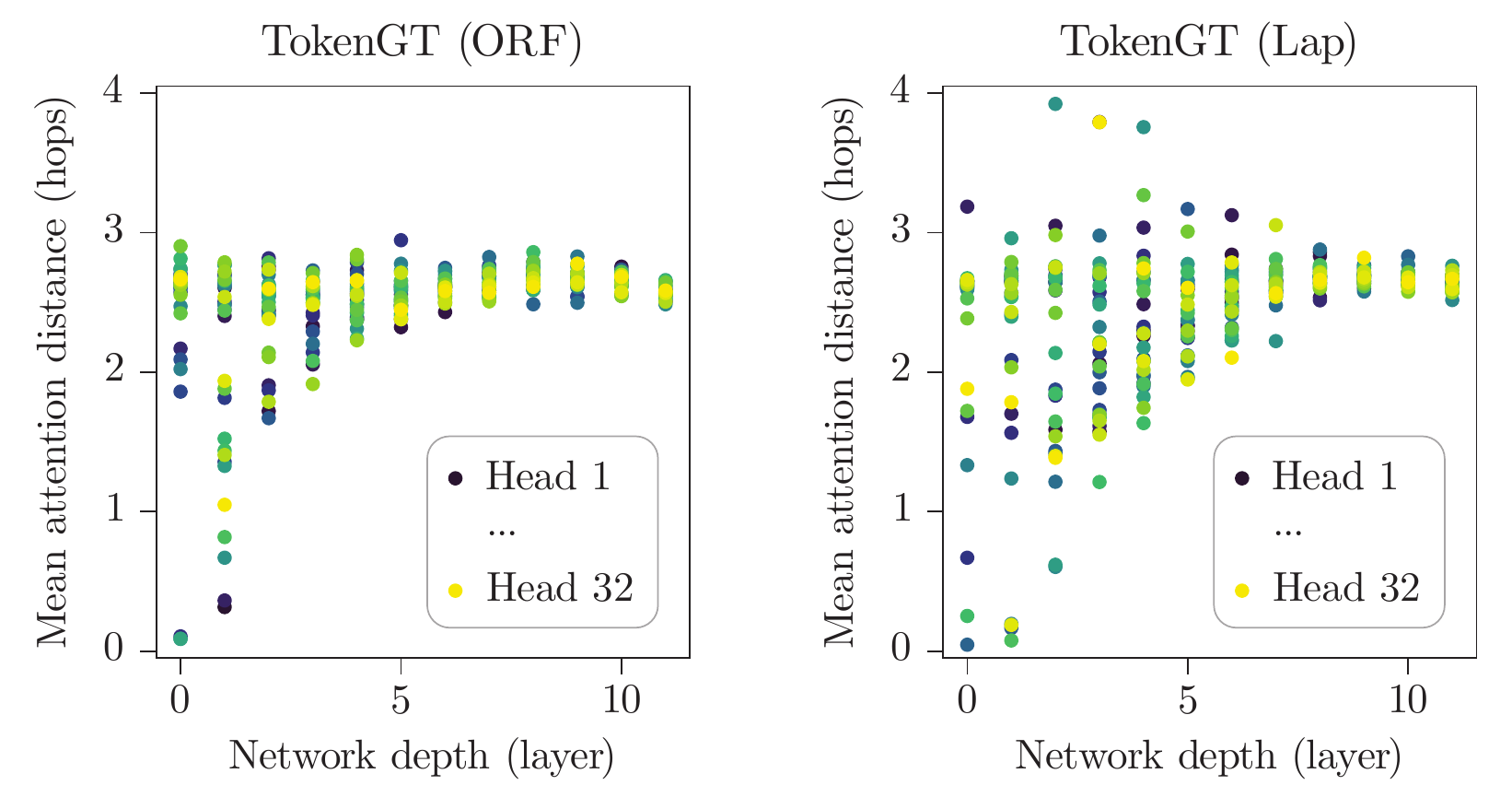}
    \vspace{-0.15cm}
    \caption{
    Attention distance by head and network depth.
    Each dot shows mean attention distance in hops across graphs of a head at a layer.
    The visualization is inspired by Dosovitskiy~et~al.~(2020)~\cite{dosovitskiy2021animage}.
    More images can be found in Appendix~\ref{sec:apdx_additional_results_pcqm4mv2}.
    }
    \label{fig:self_attention_distance}
    \vspace{-0.5cm}
\end{figure}

An exclusive characteristic of TokenGT is its minimal graph-specific inductive bias, which requires it to learn internal computation structure largely from data.
As such models are commonly known to work well with large-scale data~\cite{vaswani2017attention, dosovitskiy2021animage}, we explore the capability of TokenGT on the PCQM4Mv2 quantum chemistry regression dataset~\cite{hu2021ogb}, one of the current largest with 3.7M molecular graphs.

For TokenGT, we use both node and type identifiers, and use main Transformer encoder configuration based on Graphormer~\cite{ying2021do} with 12 layers, 768 hidden dimension, and 32 attention heads.
We try both ORF and Laplacian eigenvector as node identifiers, and denote corresponding models as \textbf{TokenGT~(ORF)} and \textbf{TokenGT~(Lap)} respectively.
As an ablation, we also experiment with the same Transformer without node and type identifiers, which we denote as \textbf{Transformer}.
Finally, we apply the kernel attention~\cite{choromanski2021rethinking} that approximates the attention computation to linear cost~(\textbf{TokenGT~(Lap)~+~Performer}).
We use AdamW optimizer with $(\beta_1, \beta_2)=(0.99, 0.999)$ and weight decay 0.1, and 60k learning rate warmup steps followed by linear decay over 1M iteration with batch size 1024.
For fine-tuning, we use 1k warmup, 0.1M training steps, and cosine learning rate decay.
We train the models on 8 RTX 3090 GPUs for 3 days. %(fine-tuning for 12 hours).
Further details are in Appendix~\ref{sec:apdx_experiment_details_pcqm4mv2}.

We provide the results in Table~\ref{table:pcqm4mv2}.
A standard Transformer on the node and edge tokens cannot recognize graph structure and shows low performance (0.2340 valid MAE).
Yet, the picture changes as soon as we augment the tokens with node and type identifiers.
Notably, TokenGT~(ORF) achieves 0.0962~MAE, which is already better than all GNN baselines.
This is a somewhat surprising result, as both ORF and the Transformer are not aware of graph structures.
This implies Transformer is strong enough to learn to interpret and reason over the incidence structure of tokens provided only implicitly by the node and type identifiers.
By further switching to Laplacian eigenvectors that encode position on graphs~\cite{dwivedi2020benchmarking}, we observe a performance boost to 0.0910~MAE, competitive to Transformers with sophisticated graph-specific modifications (\emph{e.g.}, shortest path-based spatial encoding~\cite{ying2021do}).
While such methods inject graph structure into attention matrix via bias term and therefore strictly require $\mathcal{O}(n^2)$ cost, TokenGT enables adopting kernelization for pure self-attention~\cite{choromanski2021rethinking}, resulting in TokenGT~(Lap)~+~Performer with the best performance among $\mathcal{O}(n+m)$ models (0.0935~MAE).
Further discussion on the empirical performance of TokenGT can be found in Appendix~\ref{sec:apdx_extended_discussion_pcqm4mv2}.

While our theory in Section~\ref{sec:theory} \emph{guarantees} that TokenGT can reduce to an equivariant layer by learning fixed equivariant basis at each attention head, in practice, it can freely utilize multihead self-attention to learn less restricted and more useful computation structure from data.
To analyze such a structure, we compute the attention distance across heads and network depth by averaging pairwise token distances on a graph weighted by their attention scores~(Figure~\ref{fig:self_attention_distance}).
This distance is analogous to the number of hops in message-passing.
In both TokenGT~(ORF) and TokenGT~(Lap), in the lowest layers, some heads attend globally over the graph while others consistently have small receptive fields (acting like a local message-passing operator).
In deeper layers, the attention distances increase, and most heads attend globally.
Interestingly, this behavior is highly consistent with Vision Transformers on image patches~\cite{dosovitskiy2021animage}, suggesting that hybrid architectures based on convolution to aid ViT~\cite{dai2021coatnet, yuan2021incorporating} might also work well for graphs.
While TokenGT~(ORF) shows relatively consistent attention distance over heads, TokenGT~(Lap) shows higher variance, implying that it learns more diverse attention patterns.
Judging from the higher performance of TokenGT~(Lap), this suggests that the graph structure information of the Laplacian eigenvector facilitates learning useful and diverse attention structures, which calls for future exploration of better node identifiers based on graph PEs~\cite{kreuzer2021rethinking, lim2022sign}.

\vspace{-0.2cm}
\section{Conclusion}\label{sec:conclusion}
\vspace{-0.2cm}
We showed that Transformers directly applied to graphs can work well in both theory and practice.
In the theoretical aspect, we proved that with appropriate token-wise embeddings, a Transformer on node and edge tokens is at least as expressive as $k$-IGN and $k$-WL test, making it more expressive than all message-passing GNNs.
For such token-wise embeddings, we showed that a combination of simple orthonormal node identifiers and trainable type identifiers suffices, which we also verified with a synthetic experiment.
In an experiment with PCQM4Mv2 large-scale dataset, we show that Tokenized Graph Transformer (TokenGT) performs significantly better than all GNNs and is competitive with Transformer variants with strong graph-specific architectural components~\cite{ying2021do, hussain2021edge, park2022grpe}.

While the results suggest a promising research direction, there are challenges to be addressed in future work.
First, treating each node and edge as tokens requires $\mathcal{O}((n+m)^2)$ asymptotic cost due to the quadratic nature of self-attention.
While we address this to some degree with kernelization and achieve $\mathcal{O}(n+m)$ cost, other types of efficient Transformers (\emph{e.g.}, sparse) that can deliver better performance are left to be tested.
Another issue is slightly lower performance compared to the state-of-the-art.
Adopting Transformer engineering techniques from vision and language domains, such as data scaling~\cite{brown2020language, dosovitskiy2021animage}, deepening~\cite{wang2022deepnet, wies2021which}, hybrid architectures~\cite{dai2021coatnet, yuan2021incorporating}, and self-supervision~\cite{devlin2019bert, brown2020language, he2021masked}, are promising.
In the societal aspect, to prevent the potential risky behavior in, \emph{e.g.}, decision making from graph-structured inputs, interpretability research regarding self-attention on graphs is desired.

We finish with interesting research directions that stem from our work.
As our approach advocates viewing a graph as $(n+m)$ tokens~\cite{zico2019champs}, it opens up new paradigms of graph learning, including autoregressive decoding, in-context learning, prompting, and multimodal learning.
Another interesting direction is to extend our theory and use self-attention to approximate equivariant basis for general discrete group actions, which might be a viable approach for \emph{learning equivariance from data}.

\cutparagraphup
\paragraph{Acknowledgement}
This work was supported in part by Institute of Information \& communications
Technology Planning \& Evaluation (IITP) (No. 2022-0-00926, 2022-0-00959, 2021-0-02068, and 2019-0-00075) and the National Research Foundation of Korea (NRF) (No. 2021R1C1C1012540) grants funded by the Korea government (MSIT).

{\small
\bibliography{main}
}

% arXiv
\newpage
\appendix
\section{Appendix}\label{sec:appendix}
\subsection{Proofs}\label{sec:apdx_proofs}
\subsubsection{Extended Preliminary (Cont. from Section~\ref{sec:preliminary})}\label{sec:apdx_proofs_preliminary}
Before proceeding to the proofs, we first provide additional preliminary material that supplements Section~\ref{sec:preliminary}.
We begin by formally defining multihead self-attention and Transformer.
Our definition is equivalent to Vaswani~et~al.~(2017)~\cite{vaswani2017attention}, except we omit layer normalization for simplicity as in \cite{yun2020are, hanin2017approximating, kim2021transformers}.
Specifically, a multihead self-attention layer $\textnormal{MSA}:\mathbb{R}^{n\times d}\to\mathbb{R}^{n\times d}$ is defined as:
\begin{align}
    \boldsymbol{\alpha}^h &= \textnormal{softmax}\left(\mathbf{X}w_h^Q(\mathbf{X}w_h^K)^\top/\sqrt{d_H}\right),\label{eqn:apdx_transformer_attention_coefficient}\\
    \textnormal{MSA}(\mathbf{X})_i &= \sum_{h=1}^H{\sum_{j=1}^n{\boldsymbol{\alpha}^h_{i j}\mathbf{X}_j w_h^V w_h^O}},\label{eqn:apdx_transformer_multihead_self_attention}
\end{align}
where $H$ is number of heads, $d_H$ is head size, and $w_h^Q,w_h^K\in\mathbb{R}^{d\times d_H}$, $w_h^V\in\mathbb{R}^{d\times d_v}$ $w_h^O\in\mathbb{R}^{d_v\times d}$.
In our proofs, we use biases for query and key projections as in \cite{yun2020are} but omit them here for brevity.
With multihead self-attention, a Transformer layer $\mathcal{T}:\mathbb{R}^{n\times d }\to\mathbb{R}^{n\times d}$ is defined as:
\begin{align}
    \mathbf{H} &= \mathbf{X} + \textnormal{MSA}(\mathbf{X}),\label{eqn:transformer_residual_msa}\\
    \mathcal{T}(\mathbf{X}) &= \mathbf{H} +
    \textnormal{MLP}(\mathbf{H}), \label{eqn:transformer_residual_mlp}
\end{align}
where $\textnormal{MSA}:\mathbb{R}^{n\times d}\to\mathbb{R}^{n\times d}$ is a multihead self-attention layer with $H$ heads of size $d_H$ and $\textnormal{MLP}:\mathbb{R}^{n\times d }\to\mathbb{R}^{n\times d}$ is a tokenwise MLP with hidden dimension $d_F$.

We now provide the complete definition of invariant graph networks (IGNs)~\cite{maron2019invariant, maron2019provably} and maximally expressive equivariant linear layers~\cite{maron2019invariant} summarized in Section~\ref{sec:preliminary}.
We first recall Definition~\ref{defn:invariant_graph_network} and~\ref{defn:equivariant_linear_layer}:
\begingroup
\def\thedefn{\ref{defn:invariant_graph_network}}
\begin{defn}
An order-$k$ Invariant Graph Network ($k$-IGN) is a function $F_k:\mathbb{R}^{n^k\times d_0}\to\mathbb{R}$ written as the following:
\begin{align}
    F_k = \textnormal{MLP}\circ L_{k\to 0}\circ L_{k\to k}^{(T)}\circ\sigma\circ...\circ\sigma\circ L_{k\to k}^{(1)},\tag{\ref{eqn:invariant_graph_network}}
\end{align}
where each $L_{k\to k}^{(t)}$ is equivariant linear layer~\cite{maron2019invariant} from $\mathbb{R}^{n^k\times d_{t-1}}$ to $\mathbb{R}^{n^k\times d_t}$, $\sigma$ is activation function, and $L_{k\to 0}$ is a invariant linear layer from $\mathbb{R}^{n^k\times d_T}$ to $\mathbb{R}$.
\end{defn}
\addtocounter{defn}{-1}
\endgroup
\begingroup
\def\thedefn{\ref{defn:equivariant_linear_layer}}
\begin{defn}
An equivariant linear layer is a function $L_{k\rightarrow l}:\mathbb{R}^{n^k\times d}\to\mathbb{R}^{n^l\times d'}$ written as follows for order-$k$ input $\mathbf{X}\in\mathbb{R}^{n^k\times d}$:
\begin{align}
    L_{k\rightarrow l}(\mathbf{X})_{\mathbf{i}} = \sum_{\mu}{\sum_{\mathbf{j}}{\mathbf{B}^{\mu}_{\mathbf{i}, \mathbf{j}}\mathbf{X}_{\mathbf{j}}w_{\mu}}} + \sum_{\lambda}{\mathbf{C}^{\lambda}_{\mathbf{i}}b_{\lambda}},\tag{\ref{eqn:equivariant_linear_layer}}
\end{align}
where $\mathbf{i}\in[n]^l,\mathbf{j}\in[n]^k$ are multi-indices, $w_\mu\in\mathbb{R}^{d\times d'}$, $b_\lambda\in\mathbb{R}^{d'}$ are weight and bias parameters, and $\mathbf{B}^\mu\in\mathbb{R}^{n^{l+k}}$ and $\mathbf{C}^\lambda\in\mathbb{R}^{n^l}$ are binary basis tensors corresponding to order-$(l+k)$ and order-$l$ equivalence classes $\mu$ and $\lambda$, respectively.
Invariant linear layer is a special case of $L_{k\rightarrow l}$ with $l=0$.
\end{defn}
\addtocounter{defn}{-1}
\endgroup

We now define \emph{equivalence classes} and \emph{basis tensors} mentioned briefly in Definition~\ref{defn:equivariant_linear_layer}.
The equivalence classes are defined upon a specific \emph{equivalence relation} $\sim$ on the index space of higher-order tensors as follows:
\begin{defn}\label{defn:equivalence_relation}
An order-$l$ equivalence class $\gamma\in[n]^l/_\sim$ is an equivalence class of $[n]^l$ under the equivalence relation $\sim$, where the equivalence relation $\sim$ on multi-index space $[n]^l$ relates $\mathbf{i}\sim\mathbf{j}$ if and only if $(i_1, ..., i_l) = (\pi(j_1), ..., \pi(j_l))$ for some node permutation $\pi\in S_n$.
\end{defn}
We note that a multi-index $\mathbf{i}$ has the same permutation-invariant \emph{equality pattern} to any $\mathbf{j}$ that satisfies $\mathbf{i}\sim\mathbf{j}$, \emph{i.e.}, $\mathbf{i}_{a} = \mathbf{i}_{b}\Leftrightarrow\mathbf{j}_{a} = \mathbf{j}_{b}$ for all $a, b\in [k]$.
Consequently, each equivalence class $\gamma$ in Definition~\ref{defn:equivalence_relation} is a distinct set of all order-$l$ multi-indices having a specific equality pattern.

Now, for each equivalence class, we define the corresponding \emph{basis tensor} as follows:
\begin{defn}\label{defn:basis_tensor}
An order-$l$ basis tensor $\mathbf{B}^\gamma\in\mathbb{R}^{n^l}$ corresponding to an order-$l$ equivalence class $\gamma$ is a binary tensor defined as follows:
\begin{align}\label{eqn:basis_tensor}
    \begin{array}{ll}
        \mathbf{B}_{\mathbf{i}}^{\gamma} = \left\{
        \begin{array}{cc}
            1   &  \text{\footnotesize{$\mathbf{i}\in\gamma$}} \\
            0   &  \text{\footnotesize{otherwise}}
        \end{array}\right.
    \end{array}
\end{align}
\end{defn}

For a given $l$, it is known that there exist $\textnormal{bell}(l)$ order-$l$ equivalence classes $\{\gamma_1, ..., \gamma_{\textnormal{bell}(l)}\}=[n]^l/_\sim$ regardless of $n$~\cite{maron2019invariant}.
This gives $\textnormal{bell}(l)$ order-$l$ basis tensors $\mathbf{B}^{\gamma_1}, ..., \mathbf{B}^{\gamma_{\textnormal{bell}(l)}}$ accordingly.
Thus, an equivariant linear layer $L_{k\to l}$ in Definition~\ref{defn:equivariant_linear_layer} has $\textnormal{bell}(l+k)$ weights and $\textnormal{bell}(l)$ biases.

Let us consider the first-order equivariant layer $L_{1\rightarrow 1}$ as an example.
We have $\textnormal{bell}(2)=2$ second-order equivalence classes $\gamma_1$ and $\gamma_2$ for the weight, with $\gamma_1$ the set of all $(i_1, i_2)$ with $i_1=i_2$ and $\gamma_2$ the set of all $(i_1, i_2)$ with $i_1\neq i_2$.
From Definition~\ref{defn:basis_tensor}, their corresponding basis tensors are $\mathbf{B}^{\gamma_1}=\mathbf{I}$ and $\mathbf{B}^{\gamma_2}=\mathbf{11}^\top - \mathbf{I}$.
Given a set of features $\mathbf{X}\in\mathbb{R}^{n\times d}$,
\begin{align}
    L_{1\rightarrow 1}(\mathbf{X}) = \mathbf{I}\mathbf{X}w_1+(\mathbf{11}^\top-\mathbf{I})\mathbf{X}w_2+\mathbf{1}b^{\top},
\end{align}
with two weights $w_1,w_2\in\mathbb{R}^{d\times d'}$, and a single bias $b\in\mathbb{R}^{d'}$.
For graphs ($k=l=2$), we have $\textnormal{bell}(4)=15$ weights and $\textnormal{bell}(2)=2$ biases.

\subsubsection{Proof of Lemma~\ref{lemma:approximation_equivariant_basis} (Section~\ref{sec:theoretical_results})}
To prove Lemma~\ref{lemma:approximation_equivariant_basis}, we need to show that each basis tensor $\mathbf{B}^\mu$ (Eq.~\eqref{eqn:basis_tensor}) in weights of equivariant linear layers (Eq.~\eqref{eqn:equivariant_linear_layer}) can be approximated by the self-attention coefficient $\boldsymbol{\alpha}^h$ (Eq.~\eqref{eqn:apdx_transformer_attention_coefficient}) to arbitrary precision up to normalization if its input is augmented by node and type identifiers (Section~\ref{sec:theoretical_results}).

From Definition~\ref{defn:basis_tensor}, each entry of basis tensor $\mathbf{B}_{\mathbf{i},\mathbf{j}}^\mu$ encodes whether $(\mathbf{i},\mathbf{j})\in\mu$ or not.
Here, our key idea is to break down the inclusion test $(\mathbf{i},\mathbf{j})\in\mu$ into equivalent but simpler Boolean tests that can be implemented in self-attention (Eq.~\eqref{eqn:apdx_transformer_multihead_self_attention}) as dot product of $\mathbf{i}$-th query and $\mathbf{j}$-th key followed by softmax.

To achieve this, we show some supplementary Lemmas.
We start with Lemma~\ref{lemma:equivalence_class_separation}, which comes from Lemma 1 of Kim~et~al.~(2021)~\cite{kim2021transformers} (we repeat their proof here for completeness).
\begin{lemma}\label{lemma:equivalence_class_separation}
For any order-$(l+k)$ equivalence class $\mu$, the set of all $\mathbf{i}\in[n]^l$ such that $(\mathbf{i},\mathbf{j})\in\mu$ for some $\mathbf{j}\in[n]^k$ forms an order-$l$ equivalence class.
Likewise, the set of all $\mathbf{j}$ such that $(\mathbf{i},\mathbf{j})\in\mu$ for some $\mathbf{i}$ forms an order-$k$ equivalence class.
\end{lemma}
\begin{proof}
We only prove for $\mathbf{i}$ as proof for $\mathbf{j}$ is analogous.
For some $(\mathbf{i}^1, \mathbf{j}^1)\in\mu$, let us denote the equivalence class of $\mathbf{i}^1$ as $\gamma^l$ (\emph{i.e.}, $\mathbf{i}^1\in\gamma^l$).
It is sufficient that we prove $\mathbf{i}\in\gamma^l\Leftrightarrow\exists\mathbf{j}:(\mathbf{i},\mathbf{j})\in\mu$.

($\Rightarrow$)
For all $\mathbf{i}\in\gamma^l$, as $\mathbf{i}^1\sim\mathbf{i}$, there exists some $\pi\in S_n$ such that $\mathbf{i}=\pi(\mathbf{i}^1)$ by definition.
As $\pi$ acts on multi-indices entry-wise, we have $\pi(\mathbf{i}^1,\mathbf{j}^1)=(\mathbf{i},\pi(\mathbf{j}^1))$.
As $\pi(\mathbf{i}^1,\mathbf{j}^1)\sim(\mathbf{i}^1,\mathbf{j}^1)$ holds by definition, we have $(\mathbf{i},\pi(\mathbf{j}^1))\sim(\mathbf{i}^1,\mathbf{j}^1)$, and thus $(\mathbf{i},\pi(\mathbf{j}^1))\in\mu$.
Therefore, for all $\mathbf{i}\in\gamma^l$, by setting $\mathbf{j}=\pi(\mathbf{j}^1)$ we can always obtain $(\mathbf{i}, \mathbf{j})\in\mu$.

($\Leftarrow$)
For all $(\mathbf{i},\mathbf{j})\in\mu$, as $(\mathbf{i},\mathbf{j})\sim(\mathbf{i}^1,\mathbf{j}^1)$, there exists some $\pi\in S_n$ such that $(\mathbf{i}, \mathbf{j})=\pi(\mathbf{i}^1, \mathbf{j}^1)$.
This gives $\mathbf{i}=\pi(\mathbf{i}^1)$ and $\mathbf{j}=\pi(\mathbf{j}^1)$, leading to $\mathbf{i}\sim \mathbf{i}^1$ and therefore $\mathbf{i}\in\gamma^l$.
\end{proof}

Lemma~\ref{lemma:equivalence_class_separation} states that the equivalence classes $\gamma^l$ of $\mathbf{i}$ and $\gamma^k$ of $\mathbf{j}$ are identical for all $(\mathbf{i}, \mathbf{j})\in\mu$.
Based on this, we appropriately break down the test $(\mathbf{i}, \mathbf{j})\in\mu$ into a combination of several simpler tests, in particular including $\mathbf{i}\in\gamma^l$ and $\mathbf{j}\in\gamma^k$:

\begin{lemma}\label{lemma:equivalence_class_test_breakdown}
For a given order-$(l+k)$ equivalence class $\mu$, let $\gamma^l$ and $\gamma^k$ be equivalence classes of some $\mathbf{i}^1\in[n]^l,\mathbf{j}^1\in[n]^k$ respectively that satisfies $(\mathbf{i}^1, \mathbf{j}^1)\in\mu$.
Then, for any $\mathbf{i}\in[n]^l$ and $\mathbf{j}\in[n]^k$, $(\mathbf{i},\mathbf{j})\in\mu$ holds if and only if the following conditions both hold:
\begin{enumerate}
    \item $\mathbf{i}\in\gamma^l$ and $\mathbf{j}\in\gamma^k$
    \item $\mathbf{i}_a=\mathbf{j}_b\Leftrightarrow\mathbf{i}^2_a=\mathbf{j}^2_b$ for all $a\in[l]$, $b\in[k]$, and $(\mathbf{i}^2,\mathbf{j}^2)\in\mu$
\end{enumerate}
\end{lemma}
\begin{proof}
($\Rightarrow$) If $(\mathbf{i},\mathbf{j})\in\mu$, from Lemma~\ref{lemma:equivalence_class_separation} it follows that $\mathbf{i}\in\gamma^l$ and $\mathbf{j}\in\gamma^k$.
Also, as all $(\mathbf{i}^2,\mathbf{j}^2)\in\mu$ including $(\mathbf{i},\mathbf{j})$ have the same equality pattern, it follows that for all $a\in[l]$, $b\in[k]$, and $(\mathbf{i}^2,\mathbf{j}^2)\in\mu$, if $\mathbf{i}^2_a=\mathbf{j}^2_b$ then $\mathbf{i}_a=\mathbf{j}_b$ and if $\mathbf{i}^2_a\neq\mathbf{j}^2_b$ then $\mathbf{i}_a\neq\mathbf{j}_b$.

($\Leftarrow$) We show that the conditions specify that the equivalence class of $(\mathbf{i},\mathbf{j})$ is $\mu$.

For this, it is convenient to represent an order-$l$ equivalence class $\gamma$ as an equivalent \emph{undirected graph} $\mathcal{G}=(\mathcal{V},\mathcal{E})$ defined on vertex set $\mathcal{V}=\{v_1, ..., v_l\}$ where the vertices $v_a$ and $v_b$ are connected, \emph{i.e.}, $(v_a, v_b)\in\mathcal{E}$ if and only if the equivalence class $\gamma$ specifies $\mathbf{i}_a=\mathbf{i}_b\forall \mathbf{i}\in\gamma$.
Then, for some multi-index $\mathbf{i}'\in[n]^l$, the inclusion $\mathbf{i}'\in\gamma$ holds if and only if the equivalence class of $\mathbf{i}'$ is represented as $\mathcal{G}$.

Given this, let us represent the equivalence classes $\gamma^l$, $\gamma^k$, and $\mu$ as graphs $\mathcal{G}^l$, $\mathcal{G}^k$, and $\mathcal{G}^\mu$ respectively:
\begin{align}
    \mathcal{G}^l&=(\mathcal{V}^l,\mathcal{E}^l)\textnormal{ where }\mathcal{V}^l=\{v_1, ..., v_l\},\\
    \mathcal{G}^k&=(\mathcal{V}^k,\mathcal{E}^k)\textnormal{ where }\mathcal{V}^k=\{u_1, ..., u_k\},\\
    \mathcal{G}^\mu&=(\mathcal{V}^\mu,\mathcal{E}^\mu)\textnormal{ where }\mathcal{V}^\mu=\mathcal{V}^l\cup\mathcal{V}^k=\{v_1, ..., v_l, u_1, ..., u_k\}.
\end{align}

From the precondition that $\gamma^l$ and $\gamma^k$ are equivalence classes of $\mathbf{i}^1\in[n]^l,\mathbf{j}^1\in[n]^k$ that satisfies $(\mathbf{i}^1, \mathbf{j}^1)\in\mu$, we can see that $(v_a, v_b)\in\mathcal{E}^l\Leftrightarrow (v_a, v_b)\in\mathcal{E}^\mu$ and $(u_a, u_b)\in\mathcal{E}^k\Leftrightarrow (u_a, u_b)\in\mathcal{E}^\mu$.
That is, if we consider $\mathcal{V}^l$ and $\mathcal{V}^k$ as a \emph{graph cut} of $\mathcal{G}^\mu$ and write the cut-set (edges between $\mathcal{V}^l$ and $\mathcal{V}^k$) as $\mathcal{E}^C=\{(v_a,u_b)|(v_a,u_b)\in\mathcal{E}^\mu\}$, we obtain a partition $\{\mathcal{E}^l,\mathcal{E}^k,\mathcal{E}^C\}$ of the edge set $\mathcal{E}^\mu$.

We now move to the conditions.

Let us assume the first condition that $\mathbf{i}\in\gamma^l$ and $\mathbf{j}\in\gamma^k$, with the equivalence classes represented as $\mathcal{G}^l$ and $\mathcal{G}^k$, respectively.
Now, let us consider the equivalence class of $(\mathbf{i}, \mathbf{j})$ represented by (unknown) graph $\mathcal{G}=(\mathcal{V},\mathcal{E})$.
Considering $\mathcal{V}^k$ and $\mathcal{V}^l$ as a graph cut of $\mathcal{G}$, we can see that $\mathcal{E}$ is partitioned as $\{\mathcal{E}^l,\mathcal{E}^k,\mathcal{E}^D\}$ where $\mathcal{E}^D$ is the cut-set (edges between $\mathcal{V}^l$ and $\mathcal{V}^k$).

Let us also assume the second condition $\mathbf{i}_a=\mathbf{j}_b\Leftrightarrow\mathbf{i}^2_a=\mathbf{j}^2_b$ for all $a\in[l]$, $b\in[k]$, and $(\mathbf{i}^2,\mathbf{j}^2)\in\mu$.
This directly implies that $e\in\mathcal{E}^C\Leftrightarrow e\in\mathcal{E}^D$, meaning that $\mathcal{E}^C=\mathcal{E}^D$.
As a result, we see that $\mathcal{G}$ and $\mathcal{G}^\mu$ are identical graphs, and therefore the equivalence class of $(\mathbf{i},\mathbf{j})$ is $\mu$ and $(\mathbf{i},\mathbf{j})\in\mu$ holds.

\begin{figure}[!t]
    \centering
    \includegraphics[width=0.99\textwidth]{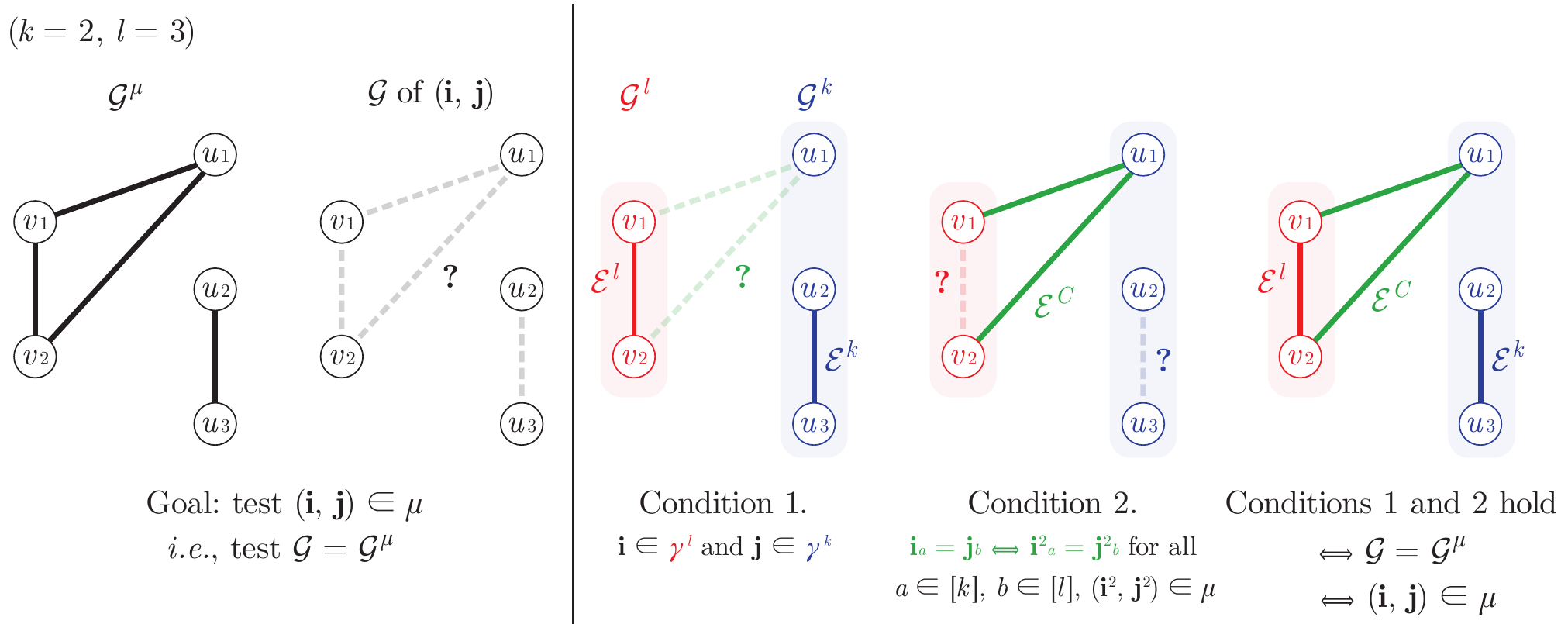}
    \caption{
    An exemplary illustration of testing $(\mathbf{i}, \mathbf{j})\in\mu$ as a combination of simpler tests, based on equivalence classes $\mu$, $\gamma^l$, and $\gamma^k$ represented as graphs $\mathcal{G}^\mu$, $\mathcal{G}^l$, and $\mathcal{G}^k$, respectively.
    }
    \label{fig:equivalence_class_graph}
\end{figure}
In Figure~\ref{fig:equivalence_class_graph}, we provide an exemplary illustration of testing $(\mathbf{i}, \mathbf{j})\in\mu$ following the above discussion.
\end{proof}

With Lemma~\ref{lemma:equivalence_class_test_breakdown}, we have a decomposition of $(\mathbf{i},\mathbf{j})\in\mu$ into independent conditions on $\mathbf{i}$ and $\mathbf{j}$ combined with pairwise conditions between $\mathbf{i}$ and $\mathbf{j}$.
In the following Definition~\ref{defn:scoring_function} and Property~\ref{property:equivalence_class_test_scoring_function}, we encode these tests into a single \emph{scoring function} that can be later implemented by self-attention.

\begin{defn}\label{defn:scoring_function}
A scoring function $\delta(\mathbf{i},\mathbf{j};\mu,\epsilon)$ is a map that, given an order-$(l+k)$ equivalence class $\mu$ and $\epsilon>0$, takes multi-indices $\mathbf{i}\in[n]^l, \mathbf{j}\in[n]^k$ and gives the following:
\begin{align}\label{eqn:scoring_function}
    \delta(\mathbf{i},\mathbf{j}; \mu, \epsilon) =
    \mathbbm{1}_{\mathbf{i}\in\gamma^l} + (1-\epsilon)\mathbbm{1}_{\mathbf{i}\notin\gamma^l} +
    \mathbbm{1}_{\mathbf{j}\in\gamma^k} + (1-\epsilon)\mathbbm{1}_{\mathbf{j}\notin\gamma^k} + \sum_{a\in[l]}\sum_{b\in[k]}\textnormal{sgn}(a, b)\mathbbm{1}_{\mathbf{i}_a=\mathbf{j}_b},
\end{align}
where $\mathbbm{1}$ is indicator, $\gamma^l$ and $\gamma^k$ are equivalence classes of $\mathbf{i}^1\in[n]^l,\mathbf{j}^1\in[n]^k$ such that $(\mathbf{i}^1, \mathbf{j}^1)\in\mu$, and the sign function $\textnormal{sgn}(\cdot,\cdot)$ is defined as follows:
\begin{align}\label{eqn:sign_function}
    \begin{array}{ll}
        \textnormal{sgn}(a, b) = \left\{
        \begin{array}{cc}
            +1   &  \mathbf{i}^2_a=\mathbf{j}^2_b\forall(\mathbf{i}^2,\mathbf{j}^2)\in\mu \\[+0.4em]
            -1   &  \mathbf{i}^2_a\neq\mathbf{j}^2_b\forall(\mathbf{i}^2,\mathbf{j}^2)\in\mu
        \end{array}\right..
    \end{array}
\end{align}
\end{defn}

An important property of the scoring function $\delta(\mathbf{i},\mathbf{j}; \mu)$ is that it gives the maximum possible value if and only if the input satisfies $(\mathbf{i},\mathbf{j})\in\mu$, as shown in the below Property~\ref{property:equivalence_class_test_scoring_function}.
\begin{property}\label{property:equivalence_class_test_scoring_function}
For given order-$(l+k)$ equivalence class $\mu$ and positive real number $\epsilon>0$, for any $\mathbf{i}\in[n]^l$ and $\mathbf{j}\in[n]^k$, $(\mathbf{i},\mathbf{j})\in\mu$ holds if and only if the scoring function $\delta(\mathbf{i},\mathbf{j};\mu,\epsilon)$ (Eq.~\eqref{eqn:scoring_function}) outputs the maximum possible value.
\end{property}
\begin{proof}
As shown in Lemma~\ref{lemma:equivalence_class_test_breakdown}, $(\mathbf{i},\mathbf{j})\in\mu$ holds if and only if the following two conditions are met.
\begin{enumerate}
    \item $\mathbf{i}\in\gamma^l$ and $\mathbf{j}\in\gamma^k$
    \item $\mathbf{i}_a=\mathbf{j}_b\Leftrightarrow\mathbf{i}^2_a=\mathbf{j}^2_b$ for all $a\in[l]$, $b\in[k]$, and $(\mathbf{i}^2,\mathbf{j}^2)\in\mu$
\end{enumerate}
When both conditions are satisfied, in Eq.~\eqref{eqn:scoring_function}, we always have $\mathbbm{1}_{\mathbf{i}\in\gamma^l} + (1-\epsilon)\mathbbm{1}_{\mathbf{i}\notin\gamma^l}=1$ and $\mathbbm{1}_{\mathbf{j}\in\gamma^k} + (1-\epsilon)\mathbbm{1}_{\mathbf{j}\notin\gamma^k}=1$.
We also have $\mathbbm{1}_{\mathbf{i}_a=\mathbf{j}_b}=1$ for $\textnormal{sgn}(a, b) = 1$ and $\mathbbm{1}_{\mathbf{i}_a=\mathbf{j}_b}=0$ for $\textnormal{sgn}(a, b) = -1$ for all $a\in[l],b\in[k]$.
As a result, Eq.~\eqref{eqn:scoring_function} gives a constant output for all $(\mathbf{i},\mathbf{j})\in\mu$.

On the other hand, if given $(\mathbf{i},\mathbf{j})$ violates any of the conditions (thus $(\mathbf{i},\mathbf{j})\notin\mu$), we either have $\mathbbm{1}_{\mathbf{i}\in\gamma^l} + (1-\epsilon)\mathbbm{1}_{\mathbf{i}\notin\gamma^l}=(1-\epsilon)$, or $\mathbbm{1}_{\mathbf{j}\in\gamma^k} + (1-\epsilon)\mathbbm{1}_{\mathbf{j}\notin\gamma^k}=(1-\epsilon)$, or $\mathbbm{1}_{\mathbf{i}_a=\mathbf{j}_b}=0$ for $\textnormal{sgn}(a, b) = 1$ or $\mathbbm{1}_{\mathbf{i}_a=\mathbf{j}_b}=1$ for $\textnormal{sgn}(a, b) = -1$ for some $a\in[l],b\in[k]$.
\emph{Any} of these violations decrements the output of Eq.~\eqref{eqn:scoring_function} by a positive (1 or $\epsilon$), resulting in a non-maximum output.

Thus, the scoring function $\delta(\mathbf{i},\mathbf{j};\mu,\epsilon)$ gives the maximum possible output if and only if $(\mathbf{i},\mathbf{j})\in\mu$.
\end{proof}

Now, we prove Lemma~\ref{lemma:approximation_equivariant_basis}.
\begingroup
\def\thelemma{\ref{lemma:approximation_equivariant_basis}}
\begin{lemma}
For all $\mathbf{X}\in\mathbb{R}^{n^k\times d}$ and their augmentation $\mathbf{X}^{in}$, self-attention coefficients $\boldsymbol{\alpha}^h$ (Eq.~\eqref{eqn:transformer_multihead_self_attention_recap}) computed with $\mathbf{X}^{in}w^{in}$ can approximate any basis tensor $\mathbf{B}^{\mu}\in\mathbb{R}^{n^{2k}}$ of order-$k$ equivariant linear layer $L_{k\to k}$ (Definition~\ref{defn:equivariant_linear_layer}) to arbitrary precision up to normalization.
\end{lemma}
\addtocounter{lemma}{-1}
\endgroup
\begin{proof}
Let us first recall the node and type identifiers (Section~\ref{sec:theoretical_results}) for order-$k$ tensors $\mathbf{X}\in\mathbb{R}^{n^k\times d}$.
Node identifier $\mathbf{P}\in\mathbb{R}^{n\times d_p}$ is an orthonormal matrix with $n$ rows, and type identifier is a trainable matrix $\mathbf{E}\in\mathbb{R}^{\textnormal{bell}(k)\times d_e}$ with $\textnormal{bell}(k)$ rows $\mathbf{E}^{\gamma_1}, ..., \mathbf{E}^{\gamma_{\textnormal{bell}(k)}}$, each designated for an order-$k$ equivalence class $\gamma$.
For each multi-index $\mathbf{i}=(i_1, ..., i_k)\in[n]^k$, we augment the corresponding input tensor entry as $[\mathbf{X}_\mathbf{i}, \mathbf{P}_{i_1}, ...,\mathbf{P}_{i_k}, \mathbf{E}^{\gamma^\mathbf{i}}]$
where $\mathbf{i}\in\gamma^\mathbf{i}$, obtaining the augmented order-$k$ tensor $\mathbf{X}^{in}\in\mathbb{R}^{n^k\times(d+kd_p+d_e)}$.
We use a trainable projection $w^{in}\in\mathbb{R}^{(d+kd_p+d_e)\times d_\mathcal{T}}$ to map them to a hidden dimension $d_\mathcal{T}$.

We now use self-attention on $\mathbf{X}^{in}w^{in}$ to perform an accurate approximation of the equivariant basis.
Specifically, we use each self-attention matrix $\boldsymbol{\alpha}^h$ (Eq.~\eqref{eqn:apdx_transformer_attention_coefficient}) to approximate each basis tensor $\mathbf{B}^{\mu_h}$ of $L_{k\to k}$ (Eq.~\eqref{eqn:equivariant_linear_layer}) to arbitrary precision up to normalization.

Let us take $d_\mathcal{T}=(d+kd_p+d_e)+\textnormal{bell}(2k)d$, putting $\textnormal{bell}(2k)d$ extra channels on top of channels of the augmented input $\mathbf{X}^{in}$.
We now let $w^{in}=[\mathbf{I}, \mathbf{0}]$, where $\mathbf{I}\in\mathbb{R}^{(d+kd_p+d_e)\times (d+kd_p+d_e)}$ is an identity matrix and $\mathbf{0}\in\mathbb{R}^{(d+kd_p+d_e)\times (d_\mathcal{T}-(d+kd_p+d_e))}$ is a matrix filled with zeros.
With this, $\mathbf{X}' = \mathbf{X}^{in}w^{in}$ simply contains $\mathbf{X}^{in}$ in the first $(d+kd_p+d_e)$ channels and zeros in the rest.

\begin{figure}[!t]
    \centering
    \includegraphics[width=0.99\textwidth]{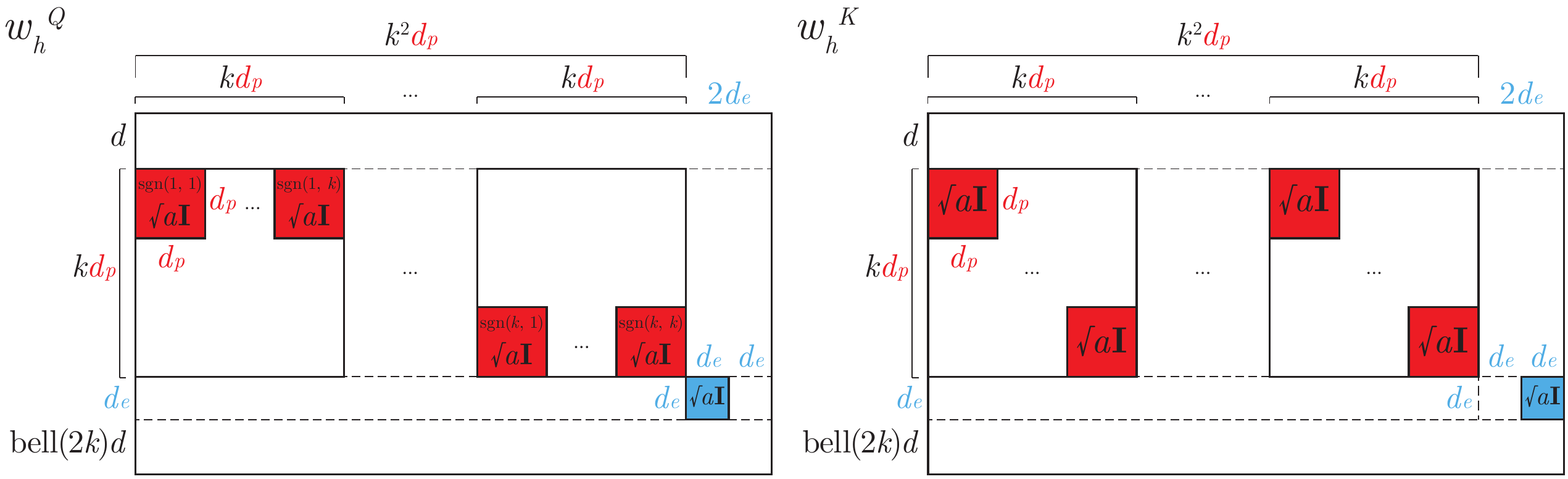}
    \caption{
    Query and key projection matrices $w_h^Q,w_h^K$ (Eq.~\eqref{eqn:whq}, Eq.~\eqref{eqn:whk}).
    Uncolored cells are zeros.
    }
    \label{fig:whqk}
\end{figure}
Now we pass $\mathbf{X}'$ to the self-attention layer in Eq.~\eqref{eqn:apdx_transformer_attention_coefficient}, where each self-attention matrix is given as $\boldsymbol{\alpha}^h=\textnormal{softmax}((\mathbf{X}'w_h^Q+b_h^Q)(\mathbf{X}'w_h^K+b_h^K)^\top/\sqrt{d_H})$.
The key idea is to set query and key projection parameters $w_h^Q,w_h^K\in\mathbb{R}^{d_\mathcal{T}\times d_H}$ and $b_h^Q,b_h^K\in\mathbb{R}^{d_H}$ appropriately so that the self-attention matrix $\boldsymbol{\alpha}^h$ approximates a given basis tensor $\mathbf{B}^\mu$ corresponding to an order-$2k$ equivalence class $\mu$.
Let $\gamma^Q$ and $\gamma^K$ be equivalence classes of some $\mathbf{i}^1,\mathbf{j}^1\in[n]^k$ respectively that satisfy $(\mathbf{i}^1,\mathbf{j}^1)\in\mu$ (see Lemma~\ref{lemma:equivalence_class_test_breakdown}).
We set head dimension $d_H = k^2d_p+2d_e$ and set $w_h^Q,w_h^K,b_h^Q,b_h^K$ as follows:
\begin{align}
    (w_h^Q)_{ij} &= \left\{
    \begin{array}{ll}
        \textnormal{sgn}(s, r)\sqrt{a}\mathbf{I}_{i-I,j-J} &
        \left\{
        \begin{array}{lll}
            \text{\footnotesize{$I<i\leq I+d_p$}} & \text{\footnotesize{for}} & \text{\footnotesize{$I=d+(s-1)d_p,$}}\\
            \text{\footnotesize{$J<j\leq J+d_p$}} & \text{\footnotesize{for}} & \text{\footnotesize{$J=(s-1)kd_p+(r-1)d_p,$}}\\
            \text{\footnotesize{for all $s,r\in[k]$}}
        \end{array}
        \right.
        \\
        \sqrt{a}\mathbf{I}_{i-I,j-J} &
        \left\{
        \begin{array}{lll}
            \text{\footnotesize{$I<i\leq I+d_e$}} & \text{\footnotesize{for}} & \text{\footnotesize{$I=d+kd_p,$}}\\
            \text{\footnotesize{$J<j\leq J+d_e$}} & \text{\footnotesize{for}} & \text{\footnotesize{$J=k^2d_p,$}}
        \end{array}
        \right.
        \\
        0 & \text{\footnotesize{otherwise}}
    \end{array}
    \right.,\label{eqn:whq}\\
    (w_h^K)_{ij} &= \left\{
    \begin{array}{ll}
        \sqrt{a}\mathbf{I}_{i-I,j-J}\hphantom{\textnormal{sgn}(s,r)} &
        \left\{
        \begin{array}{lll}
            \text{\footnotesize{$I<i\leq I+d_p$}} & \text{\footnotesize{for}} & \text{\footnotesize{$I=d+(r-1)d_p,$}}\\
            \text{\footnotesize{$J<j\leq J+d_p$}} & \text{\footnotesize{for}} & \text{\footnotesize{$J=(s-1)kd_p+(r-1)d_p,$}}\\
            \text{\footnotesize{for all $s,r\in[k]$}}
        \end{array}
        \right.
        \\
        \sqrt{a}\mathbf{I}_{i-I,j-J} &
        \left\{
        \begin{array}{lll}
            \text{\footnotesize{$I<i\leq I+d_e$}} & \text{\footnotesize{for}} & \text{\footnotesize{$I=d+kd_p,$}}\\
            \text{\footnotesize{$J<j\leq J+d_e$}} & \text{\footnotesize{for}} & \text{\footnotesize{$J=k^2d_p+d_e,$}}
        \end{array}
        \right.
        \\
        0 & \text{\footnotesize{otherwise}}
    \end{array}
    \right.,\label{eqn:whk}\\
    (b_h^Q)_j &= \left\{
    \begin{array}{ll}
        \sqrt{a}\mathbf{E}^{\gamma^K}_{j-J}\hphantom{\textnormal{sgn}(s,r)}\ \ \ \ \ \ &
        \begin{array}{lll}
             \text{\footnotesize{$J<j\leq J+d_e$}} & \text{\footnotesize{for}} & \text{\footnotesize{$J=k^2d_p$}}
        \end{array}
        \\
        0 & \text{\footnotesize{otherwise}}
    \end{array}
    \right.,\\
    (b_h^K)_j &= \left\{
    \begin{array}{ll}
        \sqrt{a}\mathbf{E}^{\gamma^Q}_{j-J}\hphantom{\textnormal{sgn}(s,r)}\ \ \ \ \ \ &
        \begin{array}{lll}
             \text{\footnotesize{$J<j\leq J+d_e$}} & \text{\footnotesize{for}} & \text{\footnotesize{$J=k^2d_p+d_e$}}
        \end{array}
        \\
        0 & \text{\footnotesize{otherwise}}
    \end{array}
    \right.,
\end{align}
where $a>0$ is a positive real, $\mathbf{I}$ is an identity matrix, and $\textnormal{sgn}(\cdot,\cdot)$ is the sign function defined in Eq.~\eqref{eqn:sign_function} (Definition~\ref{defn:scoring_function}).
In Figure~\ref{fig:whqk} we provide an illustration of the query and key weights $w_h^Q,w_h^K$.

With the parameters, $\mathbf{i}$-th query and $\mathbf{j}$-th key entries are computed as follows:
\begin{align}
    \mathbf{X}_\mathbf{i}'w_h^Q+b_h^Q &= \sqrt{a}[[\textnormal{sgn}(1,1)\mathbf{P}_{i_1}, ..., \textnormal{sgn}(1,k)\mathbf{P}_{i_1}], ..., [\textnormal{sgn}(k,1)\mathbf{P}_{i_k}, ..., \textnormal{sgn}(k,k)\mathbf{P}_{i_k}], \mathbf{E}^{\gamma^\mathbf{i}}, \mathbf{E}^{\gamma^K}],\\
    \mathbf{X}_\mathbf{j}'w_h^K+b_h^K&=\sqrt{a}[\overbrace{[\mathbf{P}_{j_1}, ..., \mathbf{P}_{j_k}], ..., [\mathbf{P}_{j_1}, ..., \mathbf{P}_{j_k}]}^{k\textnormal{ repeats}}, \mathbf{E}^{\gamma^Q}, \mathbf{E}^{\gamma^\mathbf{j}}].
\end{align}

Then, scaled pairwise dot product of query and key is given as follows:
\begin{align}\label{eqn:pairwise_dot_product}
    \frac{(\mathbf{X}_\mathbf{i}'w_h^Q+b_h^Q)^\top(\mathbf{X}_\mathbf{j}'w_h^K+b_h^K)}{\sqrt{d_H}} &= 
    \frac{a}{\sqrt{d_H}}\left(
    (\mathbf{E}^{\gamma^\mathbf{i}})^\top\mathbf{E}^{\gamma^Q} +
    (\mathbf{E}^{\gamma^\mathbf{j}})^\top\mathbf{E}^{\gamma^K} +
    \sum_{a\in[k]}\sum_{b\in[k]}\textnormal{sgn}(a,b)\mathbf{P}_{i_a}^\top\mathbf{P}_{j_b}
    \right).
\end{align}

We refer to the scaled dot product in Eq.~\eqref{eqn:pairwise_dot_product} as the \emph{unnormalized} attention coefficient $\tilde{\boldsymbol{\alpha}}^h_{\mathbf{i},\mathbf{j}}$.

\begin{figure}[!t]
    \centering
    \vspace{-0.1cm}
    \includegraphics[width=0.5\textwidth]{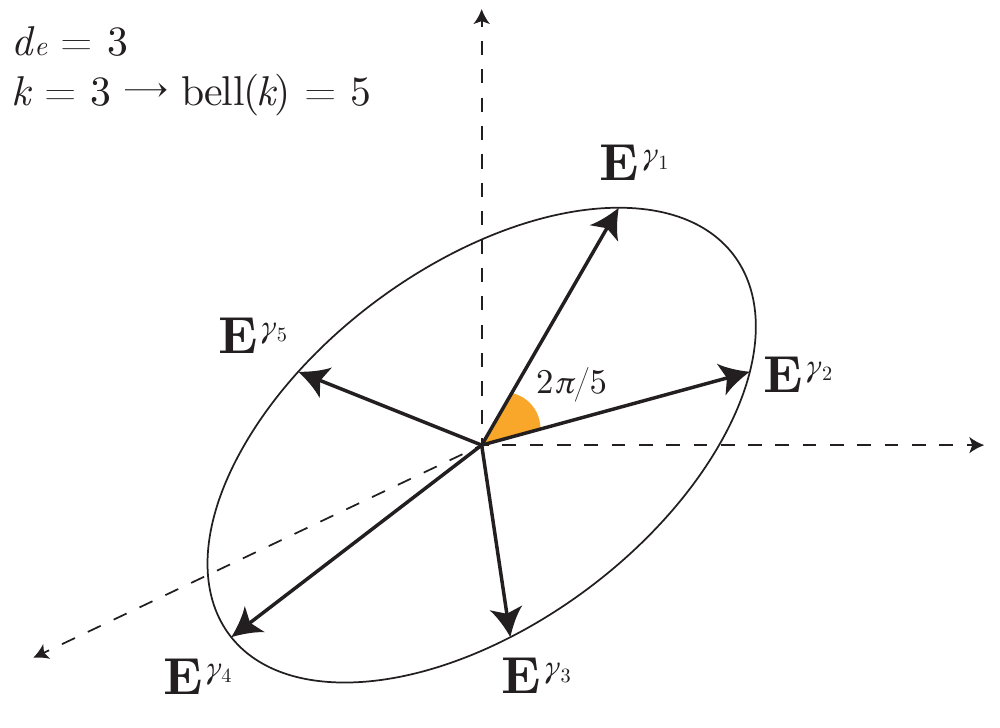}
    \vspace{-0.1cm}
    \caption{
    $k=3$ case example of $\textnormal{bell}(k)=5$ type identifiers embedded in $d_e=3$ dimensional space.
    }
    \vspace{-0.1cm}
    \label{fig:type_identifier}
\end{figure}
We now let the type identifiers $\mathbf{E}^{\gamma_1}, ..., \mathbf{E}^{\gamma_{\textnormal{bell}(k)}}$ be radially equispaced unit vectors on any two-dimensional subspace (Figure~\ref{fig:type_identifier}).
This guarantees that any pair of type identifiers $\mathbf{E}^{\gamma_1}, \mathbf{E}^{\gamma_2}$ with $\gamma_1\neq\gamma_2$ have dot product $(\mathbf{E}^{\gamma_1})^\top\mathbf{E}^{\gamma_2}\leq \cos{\left(2\pi/\textnormal{bell}(k)\right)}$.
By setting $\epsilon=1-\cos{\left(2\pi/\textnormal{bell}(k)\right)}>0$, this can be equivalently written as $(\mathbf{E}^{\gamma_1})^\top\mathbf{E}^{\gamma_2}\leq1-\epsilon$.
We additionally note that $(\mathbf{E}^{\gamma^\mathbf{i}})^\top\mathbf{E}^{\gamma^Q} = 1$ if and only if $\mathbf{i}\in\gamma^Q$ because $\gamma^\mathbf{i}=\gamma^Q\Leftrightarrow\mathbf{i}\in\gamma^Q$.

Combining the above, Eq.~\eqref{eqn:pairwise_dot_product}, and Eq.~\eqref{eqn:scoring_function}, we have the following:
\begin{align}\label{eqn:pairwise_dot_product_scoring_function}
    \tilde{\boldsymbol{\alpha}}^h_{\mathbf{i},\mathbf{j}} &= \frac{a}{\sqrt{d_H}}\delta(\mathbf{i},\mathbf{j};\mu,\epsilon)
    \ \ \text{\footnotesize{if $(\mathbf{i},\mathbf{j})\in\mu$}},\\
    \tilde{\boldsymbol{\alpha}}^h_{\mathbf{i},\mathbf{j}} &\leq \frac{a}{\sqrt{d_H}}\delta(\mathbf{i},\mathbf{j};\mu,\epsilon)
    \ \ \text{\footnotesize{otherwise}},
\end{align}
where $\epsilon=1-\cos{\left(2\pi/\textnormal{bell}(k)\right)}$ and $\delta(\mathbf{i},\mathbf{j}; \mu,\epsilon)$ is the scoring function in Eq.~\eqref{eqn:scoring_function} (Definition~\ref{defn:scoring_function}).

For a given query index $\mathbf{i}$, let us assume there exists at least one key index $\mathbf{j}$ such that $(\mathbf{i},\mathbf{j})\in\mu$~\footnote{If such key index $\mathbf{j}$ does not exist, corresponding basis tensor entries are $\mathbf{B}_{\mathbf{i},\mathbf{j}}^\mu=0\forall\mathbf{j}$, and approximation target cannot be defined as normalizing denominator $\sum_\mathbf{j}\mathbf{B}^\mu_{\mathbf{i},\mathbf{j}}$ is 0.
Thus we do not approximate for such $\mathbf{i}$, let attention row $\boldsymbol{\alpha}_{\mathbf{i},\cdot}$ have any finite values, and later silence their attention output by multiplying zero at MLP.\label{footnote:zero_row_approximation}}.
From Property~\ref{property:equivalence_class_test_scoring_function} and Eq.~\eqref{eqn:pairwise_dot_product_scoring_function}, all keys $\mathbf{j}$ that give $(\mathbf{i},\mathbf{j})\in\mu$ hold the same maximum value $\tilde{\boldsymbol{\alpha}}^h_{\mathbf{i},\mathbf{j}}= \frac{a}{\sqrt{d_H}}\delta(\mathbf{i},\mathbf{j};\mu,\epsilon)$, and any $(\mathbf{i},\mathbf{j})\notin\mu$ gives a value smaller at least by $\min{(1,\epsilon)}>0$.
Then, in softmax normalization, we send $a\to\infty$ by scaling up the query and key projection parameters.
This pushes softmax arbitrarily close to the hardmax operator, leaving only the maximal entries leading to the following:
\begin{align}\label{eqn:attention_basis_tensor_approximation}
    \boldsymbol{\alpha}^h_{\mathbf{i},\mathbf{j}} &= \frac{\exp(\tilde{\boldsymbol{\alpha}}^h_{\mathbf{i},\mathbf{j}})}
    {\sum_\mathbf{j}{\exp(\tilde{\boldsymbol{\alpha}}^h_{\mathbf{i},\mathbf{j}})}}
    \to\frac{\mathbbm{1}_{(\mathbf{i},\mathbf{j}\in\mu)}}{\sum_\mathbf{j}\mathbbm{1}_{(\mathbf{i},\mathbf{j}\in\mu)}}=\frac{\mathbf{B}^\mu_{\mathbf{i},\mathbf{j}}}{\sum_\mathbf{j}\mathbf{B}^\mu_{\mathbf{i},\mathbf{j}}}\textnormal{ as }a\to\infty.
\end{align}
Thus, as shown in Eq.~\eqref{eqn:attention_basis_tensor_approximation}, the attention coefficient $\boldsymbol{\alpha}^h$ can arbitrarily accurately approximate the normalized basis tensor $\mathbf{B}^\mu$ for given equivalence class $\mu$.
\end{proof}

\subsubsection{Proof of Theorem~\ref{thm:approximation_equivariant_layer} (Section~\ref{sec:theoretical_results})}
\begingroup
\def\thethm{\ref{thm:approximation_equivariant_layer}}
\begin{thm}
For all $\mathbf{X}\in\mathbb{R}^{n^k\times d}$ and their augmentation $\mathbf{X}^{in}$, a Transformer layer with $\textnormal{bell}(2k)$ self-attention heads that operates on $\mathbf{X}^{in}w^{in}$ can approximate an order-$k$ equivariant linear layer $L_{k\to k}(\mathbf{X})$ (Definition~\ref{defn:equivariant_linear_layer}) to arbitrary precision.
\end{thm}
\addtocounter{thm}{-1}
\endgroup
\begin{proof}
We continue from the proof of Lemma~\ref{lemma:approximation_equivariant_basis} and assume that each attention matrix $\boldsymbol{\alpha}^1,...,\boldsymbol{\alpha}^{\textnormal{bell}(2k)}$ in Eq.~\eqref{eqn:apdx_transformer_attention_coefficient} head-wise approximates each normalized basis tensor $\mathbf{B}^{\mu_1}, ..., \mathbf{B}^{\mu_{\textnormal{bell}(2k)}}$ respectively, \emph{i.e.}, $\boldsymbol{\alpha}_{\mathbf{i},\mathbf{j}}^h=\mathbf{B}^{\mu_h}_{\mathbf{i},\mathbf{j}}/\sum_\mathbf{j}\mathbf{B}^{\mu_h}_{\mathbf{i},\mathbf{j}}$.\footnote{
we handle the case $\sum_\mathbf{j}\mathbf{B}^{\mu_h}_{\mathbf{i},\mathbf{j}}=0$ later separately as mentioned in footnote~\ref{footnote:zero_row_approximation}.
\label{footnote:zero_row_approximation_recap}
}

Then, in Eq.~\eqref{eqn:apdx_transformer_multihead_self_attention} we use $d_v=d$ and set $w_h^V\in\mathbb{R}^{d_\mathcal{T}\times d}$ to $w_h^V=[\mathbf{I};\mathbf{0}]$, where $\mathbf{I}\in\mathbb{R}^{d\times d}$ is an identity matrix and $\mathbf{0}\in\mathbb{R}^{(d_\mathcal{T}-d)\times d}$ is a matrix filled with zeros.
With this, the value projection of each $\mathbf{i}$-th entry simply gives the original input features, $\mathbf{X}_\mathbf{i}'w_h^V=\mathbf{X}_\mathbf{i}$.

Then, we set output projections $w_h^O\in\mathbb{R}^{d\times d_\mathcal{T}}$ as follows:
\begin{align}
    (w_h^O)_{ij}= \left\{\begin{array}{ll}
            (w_{\mu_h})_{i,j-J} & \text{\footnotesize{$J<j\leq J+d$ for $J=(d+kd_p+d_e)+(h-1)d$}} \\
            0 & \text{\footnotesize{otherwise}}
        \end{array}\right.,
\end{align}
where $w_{\mu_1}, ..., w_{\mu_{\textnormal{bell}(2k)}}\in\mathbb{R}^{d\times d}$ are weight matrices of the given equivariant linear layer $L_{k\to k}$ in Eq.~\eqref{eqn:equivariant_linear_layer} (Definition~\ref{defn:equivariant_linear_layer}), each corresponding to equivalence classes $\mu_1,...,\mu_{\textnormal{bell}(2k)}$.

Then, output projection applied after value projection of each $\mathbf{i}$-th input entry gives the following:
\begin{align}
    \mathbf{X}_\mathbf{i}'w_h^Vw_h^O=\mathbf{X}_\mathbf{i}w_h^O=[\mathbf{0},\mathbf{0}_L,\mathbf{X}_\mathbf{i}w_{\mu_h},\mathbf{0}_R],
\end{align}
where $\mathbf{0}\in\mathbb{R}^{(d+kd_p+d_e)}$, $\mathbf{0}_L\in\mathbb{R}^{(h-1)d}$, $\mathbf{0}_R\in\mathbb{R}^{d_\mathcal{T}-(d+kd_p+d_e)-hd}$ are zero vectors.

Based on the results, we compute the MSA with skip connection $\mathbf{H} = \mathbf{X}' + \textnormal{MSA}(\mathbf{X}')$ (Eq.~\eqref{eqn:transformer_residual_msa}):
\begin{align}
    \mathbf{H}_\mathbf{i}
    &=  \mathbf{X}'_\mathbf{i} + \textnormal{MSA}(\mathbf{X}')_\mathbf{i}\\
    &= \left[
    \mathbf{X}_\mathbf{i},\mathbf{P}_{i_1},...,\mathbf{P}_{i_k},\mathbf{E}^{\gamma^\mathbf{i}},\mathbf{0}_1
    \right]\nonumber\\
    &+ \left[
    \mathbf{0}_2,
    \sum_\mathbf{j}{
    \frac{\mathbf{B}_{\mathbf{i},\mathbf{j}}^{\mu_1}}
    {\sum_\mathbf{j}\mathbf{B}_{\mathbf{i},\mathbf{j}}^{\mu_1}}
    \mathbf{X}_\mathbf{j}w_{\mu_1}
    },
    ...,
    \sum_\mathbf{j}{
    \frac{\mathbf{B}_{\mathbf{i},\mathbf{j}}^{\mu_{\textnormal{bell}(2k)}}}
    {\sum_\mathbf{j}\mathbf{B}_{\mathbf{i},\mathbf{j}}^{\mu_{\textnormal{bell}(2k)}}}
    \mathbf{X}_\mathbf{j}w_{\mu_{\textnormal{bell}(2k)}}
    }
    \right]\\
    & = \left[
    \mathbf{X}_\mathbf{i},\mathbf{P}_{i_1},...,\mathbf{P}_{i_k},\mathbf{E}^{\gamma^\mathbf{i}},
    \sum_\mathbf{j}{
    \frac{\mathbf{B}_{\mathbf{i},\mathbf{j}}^{\mu_1}}
    {\sum_\mathbf{j}\mathbf{B}_{\mathbf{i},\mathbf{j}}^{\mu_1}}
    \mathbf{X}_\mathbf{j}w_{\mu_1}
    },
    ...,
    \sum_\mathbf{j}{
    \frac{\mathbf{B}_{\mathbf{i},\mathbf{j}}^{\mu_{\textnormal{bell}(2k)}}}
    {\sum_\mathbf{j}\mathbf{B}_{\mathbf{i},\mathbf{j}}^{\mu_{\textnormal{bell}(2k)}}}
    \mathbf{X}_\mathbf{j}w_{\mu_{\textnormal{bell}(2k)}}
    }
    \right],
\end{align}
where $\mathbf{0}_1\in\mathbb{R}^{d_\mathcal{T}-(d+kd_p+d_e)}$, $\mathbf{0}_2\in\mathbb{R}^{(d+kd_p+d_e)}$ are zero vectors.

We use feedforward MLP (Eq.~\eqref{eqn:transformer_residual_mlp}) to denormalize and combine the result.
Specifically, we make the elementwise MLP approximate following $f:\mathbb{R}^{d_\mathcal{T}}\to\mathbb{R}^{d_\mathcal{T}}$ based on universal approximation~\cite{hanin2017approximating, hornik1989multilayer}:
\begin{align}
    f(\mathbf{H}_\mathbf{i})_j &= 
    \left\{\begin{array}{ll}
        - \mathbf{H}_{\mathbf{i},j}
        + \sum_{h\in[\textnormal{bell}(2k)]}{g(\mathbf{H}_\mathbf{i})_h\mathbf{H}_{\mathbf{i},j+J}}
        + b(\mathbf{H}_\mathbf{i})_j
        & \text{\footnotesize{$j \leq d$}} \\
        0
        & \text{\footnotesize{$d < j \leq (d + kd_p + d_e)$}} \\
        -\mathbf{H}_{\mathbf{i},j}
        & \text{\footnotesize{$j > (d + kd_p + d_e)$}}
    \end{array}\right.,\label{eqn:elementwise_mlp_approximation}\\
    g(\mathbf{H}_\mathbf{i})_h &= \sum_{\mathbf{j}}\mathbf{B}_{\mathbf{i},\mathbf{j}}^{\mu_h},\\
    b(\mathbf{H}_\mathbf{i})_j &= (b_{\gamma^{\mathbf{i}}})_j = (\sum_\gamma{\mathbf{C}_\mathbf{i}^\gamma}b_\gamma)_j,
\end{align}
where $J = (d+kd_p+d_e)+(h-1)d$, and $b_{\gamma_1},...,b_{\gamma_{\textnormal{bell}(k)}}$ are biases of the given equivariant linear layer $L_{k\to k}$ with corresponding basis tensors $\mathbf{C}^{\gamma_1}, ..., \mathbf{C}^{\gamma_{\textnormal{bell}(k)}}$ (Eq.~\eqref{eqn:equivariant_linear_layer}).

Within the function $f$, the auxiliary function $g:\mathbb{R}^{d_\mathcal{T}}\to\mathbb{R}^{\textnormal{bell}(2k)}$ computes head-wise attention denormalization factor\footnote{
Note that the $g(\mathbf{H}_\mathbf{i})_h$ gives $0$ for all $\mathbf{i}$ that $\sum_\mathbf{j}\mathbf{B}^{\mu_h}_{\mathbf{i},\mathbf{j}}=0$, which automatically handles the corner case as discussed at footnote~\ref{footnote:zero_row_approximation} and footnote~\ref{footnote:zero_row_approximation_recap}.
}
and $b:\mathbb{R}^{d_\mathcal{T}}\to\mathbb{R}^d$ computes bias.
As $n$ and $k$ are fixed constants, the outputs $g(\mathbf{H}_\mathbf{i})$ and $b(\mathbf{H}_\mathbf{i})$ only depend on the equivalence class $\gamma^\mathbf{i}$ of $\mathbf{i}$.
We note that the functions $g$ and $b$ can deduce the equivalence class from the input $\mathbf{H}_\mathbf{i}$, by extracting the type identifier $\mathbf{E}^{\gamma^{\mathbf{i}}}=\mathbf{H}_\mathbf{i}^\top[\mathbf{0}_3,\mathbf{I},\mathbf{0}_4]$ with $\mathbf{I}\in\mathbb{R}^{d_e\times d_e}$ an identity matrix and $\mathbf{0}_3\in\mathbb{R}^{d+kd_p}$, $\mathbf{0}_4\in\mathbb{R}^{\textnormal{bell}(2k)d}$ zero matrices.

Based on the results, we compute the feedforward MLP with skip connection $\mathcal{T}(\mathbf{X}')=\mathbf{H} + \textnormal{MLP}(\mathbf{H})$ (Eq.~\eqref{eqn:transformer_residual_mlp}), which is the output of Transformer layer $\mathcal{T}$:
\begin{align}
    \mathcal{T}(\mathbf{X}')_\mathbf{i}
    &= \mathbf{H}_\mathbf{i} + \textnormal{MLP}(\mathbf{H})_\mathbf{i}\\
    &= \mathbf{H}_\mathbf{i} + f(\mathbf{H}_\mathbf{i})\\
    &= \left[
    \mathbf{X}_\mathbf{i},
    \mathbf{P}_{i_1},...,\mathbf{P}_{i_k},
    \mathbf{E}^{\gamma^\mathbf{i}},
    \mathbf{S}_\mathbf{i}^1,
    ...,
    \mathbf{S}_\mathbf{i}^{\textnormal{bell}(2k)}
    \right]\nonumber\\
    &+ \left[
    -\mathbf{X}_\mathbf{i} + \sum_{h\in[\textnormal{bell}(2k)]}\sum_\mathbf{j}\mathbf{B}^{\mu_h}_{\mathbf{i},\mathbf{j}}\mathbf{X}_\mathbf{j}w_{\mu_h}
    + \sum_\gamma\mathbf{C}^\gamma_\mathbf{i}b_\gamma,
    \mathbf{0}_5,
    - \mathbf{S}_\mathbf{i}^1,
    ...,
    - \mathbf{S}_\mathbf{i}^{\textnormal{bell}(2k)}
    \right],\\
    &= \left[ \sum_{h\in[\textnormal{bell}(2k)]}\sum_\mathbf{j}\mathbf{B}^{\mu_h}_{\mathbf{i},\mathbf{j}}\mathbf{X}_\mathbf{j}w_{\mu_h}
    + \sum_\gamma\mathbf{C}^\gamma_\mathbf{i}b_\gamma,
    \mathbf{P}_{i_1},...,\mathbf{P}_{i_k},
    \mathbf{E}^{\gamma^\mathbf{i}},
    \mathbf{0}_6
    \right],\label{eqn:transformer_layer_output}
\end{align}
where we write $\mathbf{S}_\mathbf{i}^h = \sum_\mathbf{j}{\frac{\mathbf{B}_{\mathbf{i},\mathbf{j}}^{\mu_h}}{\sum_\mathbf{j}\mathbf{B}_{\mathbf{i},\mathbf{j}}^{\mu_h}}\mathbf{X}_\mathbf{j}w_{\mu_h}}$ and $\mathbf{0}_5\in\mathbb{R}^{kd_p+d_e}, \mathbf{0}_6\in\mathbb{R}^{(d_\mathcal{T}-(d+kd_p+d_e))}$ are zeros.

In Eq.~\eqref{eqn:transformer_layer_output}, note that the Transformer layer $\mathcal{T}(\mathbf{X}')_\mathbf{i}$ only updates the first $d$ channels of $\mathbf{X}'_\mathbf{i}$ from $\mathbf{X}_\mathbf{i}$ to $\sum_{\mu}\sum_\mathbf{j}\mathbf{B}^{\mu}_{\mathbf{i},\mathbf{j}}\mathbf{X}_\mathbf{j}w_{\mu} + \sum_\gamma\mathbf{C}^\gamma_\mathbf{i}b_\gamma$.
Therefore, with a simple projection $w^{out} = [\mathbf{I};\mathbf{0}]\in\mathbb{R}^{d_\mathcal{T}\times d}$ where $\mathbf{I}\in\mathbb{R}^{d\times d}$ is an identity matrix and $\mathbf{0}\in\mathbb{R}^{(d_\mathcal{T}-d)\times d}$ is a matrix filled with zeros, we can select the first $d$ channels of the output and finally obtain $\mathcal{T}(\mathbf{X}')w^{out}=L_{k\to k}(\mathbf{X})$.

In conclusion, a Transformer layer with $\textnormal{bell}(2k)$ self-attention heads that operates on augmented $\mathbf{X}'$ can approximate any given $L_{k\to k}(\mathbf{X})$ to arbitrary precision.
\end{proof}

\subsubsection{Proof of Theorem~\ref{thm:approximation_order_k_ign} (Section~\ref{sec:theoretical_results})}
\begingroup
\def\thethm{\ref{thm:approximation_order_k_ign}}
\begin{thm}
For all $\mathbf{X}\in\mathbb{R}^{n^k\times d}$ and their augmentation $\mathbf{X}^{in}$, a Transformer composed of $T$ layers that operates on $\mathbf{X}^{in}w^{in}$ followed by sum-pooling and MLP can approximate an $k$-IGN $F_k(\mathbf{X})$ (Definition~\ref{defn:invariant_graph_network}) to arbitrary precision.
\end{thm}
\addtocounter{thm}{-1}
\endgroup
\begin{proof}
We continue from the proof of Theorem~\ref{thm:approximation_equivariant_layer}, and assume that each Transformer layer $\mathcal{T}$ can approximate a given $L_{k\to k}$ by only updating the first $d$ channels.

Then, based on Theorem~\ref{thm:approximation_equivariant_layer} we assume the following for each $t<T$:
\begin{align}
    \mathcal{T}^{(t)}(\mathbf{X}')_\mathbf{i} &=
    \left[
    \sigma(L_{k\to k}^{(t)}(\mathbf{X}))_\mathbf{i},
    \mathbf{P}_{i_1},...,\mathbf{P}_{i_k},
    \mathbf{E}^{\gamma^\mathbf{i}},
    \mathbf{0}_6
    \right]
\end{align}
where $\mathbf{X}'_\mathbf{i}=[\mathbf{X}_\mathbf{i}, \mathbf{P}_{i_1}, ...,\mathbf{P}_{i_k}, \mathbf{E}^{\gamma^\mathbf{i}},\mathbf{0}_6]$.
While Theorem~\ref{thm:approximation_equivariant_layer} gives $L_{k\to k}^{(t)}(\mathbf{X})$ in the first $d$ channels, we add elementwise activation $\sigma(\cdot)$ by absorbing it into the elementwise MLP in Eq.~\eqref{eqn:elementwise_mlp_approximation}.
Then, leveraging the property that each Transformer layer $\mathcal{T}^{(t)}$ only updates the first $d$ channels, we stack $T-1$ Transformer layers $\mathcal{T}^{(1)}$, ..., $\mathcal{T}^{(T-1)}$ and obtain the following:
\begin{align}
    \mathcal{T}^{(T-1)}\circ...\circ\mathcal{T}^{(1)}(\mathbf{X}')_\mathbf{i} =
    \left[
    \sigma\circ L_{k\to k}^{(T-1)}\circ\sigma\circ...\circ\sigma\circ L_{k\to k}^{(1)}(\mathbf{X})_\mathbf{i},
    \mathbf{P}_{i_1},...,\mathbf{P}_{i_k},
    \mathbf{E}^{\gamma^\mathbf{i}},
    \mathbf{0}_6
    \right].\label{eqn:ign_1_to_t_minus_1}
\end{align}

For the last layer $\mathcal{T}^{(T)}$, we follow the procedure in the proof of Theorem~\ref{thm:approximation_equivariant_layer} to approximate $L_{k\to k}^{(T)}$, but slightly tweak Eq.~\eqref{eqn:elementwise_mlp_approximation} so that elementwise MLP copies each output entry $L_{k\to k}(\mathbf{X})^{(T)}_\mathbf{i}$ in appropriate reserved channels.
Specifically, we let the elementwise MLP approximate following $f'$:
\begin{align}
    f'(\mathbf{H}_\mathbf{i})_j &= 
    \left\{\begin{array}{ll}
        - \mathbf{H}_{\mathbf{i},j}
        & \text{\footnotesize{$j \leq D$}} \\
        -\mathbf{H}_{\mathbf{i},j} + \mathbf{C}_\mathbf{i}^{\gamma_a}\mathbf{F}_{\mathbf{i},j-(D+(a-1)d)}
        & \text{\footnotesize{$D + (a-1)d < j \leq D + ad$ for all $a\in[\textnormal{bell}(k)]$}} \\
        -\mathbf{H}_{\mathbf{i},j}
        & \text{\footnotesize{$D + \textnormal{bell}(k)d < j$}}
    \end{array}\right.,\label{eqn:ign_last_elementwise_mlp_approximation}
\end{align}
where $D=(d + kd_p + d_e)$ and we abbreviate $\mathbf{F}_{\mathbf{i},j} = \sum_{h\in[\textnormal{bell}(2k)]}{g(\mathbf{H}_\mathbf{i})_h\mathbf{H}_{\mathbf{i},j+J}}+ b(\mathbf{H}_\mathbf{i})_j$ with $J, g, b$ defined as same as in Eq.~\eqref{eqn:elementwise_mlp_approximation}.
Recall that $\mathbf{C}^{\gamma_a}_\mathbf{i}=1$ if and only if $\mathbf{i}\in\gamma_a$.
Therefore, with Eq.~\eqref{eqn:ign_last_elementwise_mlp_approximation}, we are simply duplicating each output entry $\mathbf{F}_{\mathbf{i}}=L_{k\to k}^{(T)}(\mathbf{X})_\mathbf{i}$ to spare channel indices reserved for the equivalence class of $\mathbf{i}$ ($\gamma_a$ that $\mathbf{i}\in\gamma_a$).

With the choice of $\mathcal{T}^{(T)}$, the layer output $\mathcal{T}^{(T)}(\mathbf{X}')=\mathbf{H} + \textnormal{MLP}(\mathbf{H})$ (Eq.~\eqref{eqn:transformer_residual_mlp}) is computed as:
\begin{align}
    \mathcal{T}^{(T)}(\mathbf{X}')_\mathbf{i}
    &= \mathbf{H}_\mathbf{i} + \textnormal{MLP}(\mathbf{H})_\mathbf{i}\\
    &= \mathbf{H}_\mathbf{i} + f'(\mathbf{H}_\mathbf{i})\\
    &= \left[\mathbf{0}_7,
    \mathbf{C}_\mathbf{i}^{\gamma_1}L_{k\to k}^{(T)}(\mathbf{X})_\mathbf{i},
    ...,
    \mathbf{C}_\mathbf{i}^{\gamma_\textnormal{bell}(k)}L_{k\to k}^{(T)}(\mathbf{X})_\mathbf{i},
    \mathbf{0}_8
    \right],\label{eqn:ign_t}
\end{align}
where $\mathbf{0}_7\in\mathbb{R}^{(d+kd_p+d_e)}, \mathbf{0}_8\in\mathbb{R}^{d_\mathcal{T}-(d+kd_p+d_e)-\textnormal{bell}(k)d}$ are zero vectors.

Then, by applying $\mathcal{T}^{(T)}$ (Eq.~\eqref{eqn:ign_t}) on top of $\mathcal{T}^{(T-1)}\circ...\circ\mathcal{T}^{(1)}$ (Eq.~\eqref{eqn:ign_1_to_t_minus_1}), we obtain the following:
\begin{align}
    \mathcal{T}^{(T)}\circ...\circ\mathcal{T}^{(1)}(\mathbf{X}')_\mathbf{i} =
    \left[
    \mathbf{0}_7,
    \mathbf{C}_\mathbf{i}^{\gamma_1}\mathbf{Y}_\mathbf{i},
    ...,
    \mathbf{C}_\mathbf{i}^{\gamma_\textnormal{bell}(k)}\mathbf{Y}_\mathbf{i},
    \mathbf{0}_8
    \right].
\end{align}
where we abbreviate $\mathbf{Y} = L_{k\to k}^{(T)}\circ\sigma\circ...\circ\sigma\circ L_{k\to k}^{(1)}(\mathbf{X})$.

The remaining step is to utilize $\textnormal{MLP}\circ\textnormal{sumpool}$ to approximate $\textnormal{MLP}_k\circ L_{k\to 0}$ that tops $F_k$.
By sum-pooling over all indices $\mathbf{i}$, we obtain the following:
\begin{align}
    \textnormal{sumpool}\circ\mathcal{T}^{(T)}\circ...\circ\mathcal{T}^{(1)}(\mathbf{X}') =
    \left[
    \mathbf{0}_7,
    \sum_\mathbf{i}\mathbf{C}_\mathbf{i}^{\gamma_1}\mathbf{Y}_\mathbf{i},
    ...,
    \sum_\mathbf{i}\mathbf{C}_\mathbf{i}^{\gamma_\textnormal{bell}(k)}\mathbf{Y}_\mathbf{i},
    \mathbf{0}_8
    \right].\label{eqn:sum_pool}
\end{align}

Now, we let the final $\textnormal{MLP}$ approximate the following function $f'':\mathbb{R}^{d_\mathcal{T}}\to\mathbb{R}^d$:
\begin{align}
    f''(\mathbf{X}) = \textnormal{MLP}_k\left(\sum_{a\in[\textnormal{bell}(k)]}\mathbf{X}^aw_{\mu_a}+b_f\right)\textnormal{ where }\mathbf{X}^{a}_j = \mathbf{X}_{D+(a-1)d+j}\textnormal{ for }j\in[d],
\end{align}
where $w_{\mu_1}, ..., w_{\mu_{\textnormal{bell}(k)}}\in\mathbb{R}^{d\times d}$ and $b_f\in\mathbb{R}^d$ are the weights and bias of the given invariant linear layer $L_{k\to 0}$, and each $\mathbf{X}^a\in\mathbb{R}^d$ is a chunk that coincides with reserved channels in Eq.~\eqref{eqn:ign_last_elementwise_mlp_approximation}.
By plugging in the sum-pooled representation in Eq.~\eqref{eqn:sum_pool}, we finally obtain the following:
\begin{align}
    \textnormal{MLP}\circ\textnormal{sumpool}\circ\mathcal{T}^{(T)}\circ...\circ\mathcal{T}^{(1)}(\mathbf{X}') &= f''\circ\textnormal{sumpool}\circ\mathcal{T}^{(T)}\circ...\circ\mathcal{T}^{(1)}(\mathbf{X}')\\
    &= 
    \textnormal{MLP}_k\left(\sum_{a\in[\textnormal{bell}(k)]}\sum_\mathbf{i}\mathbf{C}_\mathbf{i}^{\gamma_a}\mathbf{Y}_\mathbf{i}w_{\mu_a}+b_f\right)\\
    &= \textnormal{MLP}_k\circ L_{k\to 0}(\mathbf{Y})\\
    &= \textnormal{MLP}_k\circ L_{k\to 0}\circ L_{k\to k}^{(T)}\circ\sigma\circ...\circ\sigma\circ L_{k\to k}^{(1)}(\mathbf{X})\\
    &= F_k(\mathbf{X}),
\end{align}
where the last equality comes from Definition~\ref{defn:invariant_graph_network}.

Taken together, we arrive at the conclusion that $\textnormal{MLP}\circ\textnormal{sumpool}\circ\mathcal{T}^{(T)}\circ...\circ\mathcal{T}^{(1)}(\mathbf{X}')$ can approximate $F_k(\mathbf{X})$ to arbitrary precision.
\end{proof}

\subsection{Additional Discussion on Linear Attention for Graph Transformers (Section~\ref{sec:related_work})}\label{sec:apdx_extended_related_work}
We provide an additional discussion on related work, specifically on why Graphormer~\cite{ying2021do}, based on fully-connected self-attention on nodes, is not compatible with many linear attention methods that reduce the memory complexity from $\mathcal{O}(n^2)$ to $\mathcal{O}(n)$.
A range of prior graph Transformers including EGT~\cite{hussain2021edge}, GRPE~\cite{park2022grpe}, and SAN~\cite{kreuzer2021rethinking} can be analyzed analogously.
Let us first remind self-attention with query, key, value $\mathbf{Q},\mathbf{K},\mathbf{V}\in\mathbb{R}^{n\times d}$ and self-attention matrix $\boldsymbol{\alpha}\in\mathbb{R}^{n\times n}$:
\begin{align}
    \text{Att}(\mathbf{Q},\mathbf{K},\mathbf{V})_i = \sum_j\boldsymbol{\alpha}_{ij}\mathbf{V}_j\text{ where }\boldsymbol{\alpha}_{ij} = \frac{\text{exp}(\mathbf{Q}_i^\top\mathbf{K}_j/\sqrt{d})}{\sum_k\text{exp}(\mathbf{Q}_i^\top\mathbf{K}_k/\sqrt{d})}.
\end{align}

For graphs, as self-attention on nodes alone cannot recognize the edge connectivity, Graphormer incorporates the structural information of an input graph $\mathcal{G}$ into the self-attention matrix $\boldsymbol{\alpha}^\mathcal{G}\in\mathbb{R}^{n\times n}$ via attention bias matrix $\mathbf{b}^\mathcal{G}\in\mathbb{R}^{n\times n}$ (referred to as the edge and spatial encoding) as the following:
\begin{align}
    \boldsymbol{\alpha}^\mathcal{G}_{ij} = \frac{\text{exp}(\mathbf{Q}_i^\top\mathbf{K}_j/\sqrt{d} +\mathbf{b}_{ij}^\mathcal{G})}{\sum_k\text{exp}(\mathbf{Q}_i^\top\mathbf{K}_k/\sqrt{d} +\mathbf{b}_{ij}^\mathcal{G})}.\label{eqn:self_attention_bias}
\end{align}

Unfortunately, this modification immediately precludes the adaptation of many efficient attention techniques developed for pure self-attention.
As representative examples, we take Performer~\cite{choromanski2021rethinking}, Linear Transformer~\cite{katharopoulos2020transformers}, Efficient Transformer~\cite{shen2021efficient}, and Random Feature Attention~\cite{peng2021random}.
The methods are based on kernelization of the $\text{Att}(\cdot)$ operator as the following:
\begin{align}
    \text{Att}_\phi(\mathbf{Q},\mathbf{K},\mathbf{V})_i &= \sum_j\frac{\phi(\mathbf{Q}_i)^\top\phi(\mathbf{K}_j)}{\sum_k\phi(\mathbf{Q}_i)^\top\phi(\mathbf{K}_k)}\mathbf{V}_j = \frac{\phi(\mathbf{Q}_i)^\top\left(\sum_j\phi(\mathbf{K}_j)\mathbf{V}_j^\top\right)}{\phi(\mathbf{Q}_i)^\top\left(\sum_k\phi(\mathbf{K}_k)\right)}.\label{eqn:kernel_attention}
\end{align}
As the above factorization of $\text{exp}(\cdot)$ into a pairwise dot product eliminates the need to explicitly compute the attention matrix, it reduces both time and memory cost of self-attention to $\mathcal{O}(n)$.
Yet, in Eq.~\eqref{eqn:self_attention_bias}, since the bias $\mathbf{b}_{ij}^\mathcal{G}$ is added to the dot product \emph{before} $\text{exp}(\cdot)$, it is required that the full attention matrix $\boldsymbol{\alpha}^\mathcal{G}\in\mathbb{R}^{n\times n}$ is always explicitly computed.
Thus, Graphormer and related variations are unable to utilize the method and are bound to $\mathcal{O}(n^2)$.

While above discussion regards kernelization, a wide range of other efficient Transformers, including Set Transformer~\cite{lee2019set}, LUNA~\cite{ma2021luna}, Linformer~\cite{wang2020linformer}, Nyströmformer~\cite{xiong2021nystromformer}, Perceiver~\cite{andrew2021perceiver}, and Perceiver-IO~\cite{andrew2021perceiverio} are not applicable to Graphormer due to similar reasons.

\subsection{Experimental Details (Section~\ref{sec:experiments})}\label{sec:apdx_experimental_details}
We provide detailed information on the datasets and models used in our experiments in Section~\ref{sec:experiments}.
Dataset statistics can be found in Table~\ref{table:dataset_statistics}.

\subsubsection{Implementation Details of Node and Type Identifiers}\label{sec:apdx_node_type_identifiers}
In most of our experiments on graph data $(k=2)$, we fix the Transformer encoder configuration and experiment with choices of node identifiers $\mathbf{P}\in\mathbb{R}^{n\times d_p}$ and type identifiers $\mathbf{E}\in\mathbb{R}^{2\times d_e}$ (Section~\ref{sec:methods}).

For type identifiers $\mathbf{E}$, we set $d_e$ equal to the hidden dimension $d$ of the main encoder $d_e=d$ and initialize and train them jointly with the model.

For orthonormal node identifiers $\mathbf{P}$, we use normalized orthogonal random features (ORFs) or Laplacian eigenvectors obtained as follows:
\begin{itemize}
    \item For orthogonal random features (ORFs), we use rows of random orthogonal matrix $\mathbf{Q}\in\mathbb{R}^{n\times n}$ obtained with QR decomposition of random Gaussian matrix $\mathbf{G}\in\mathbb{R}^{n\times n}$~\cite{yu2016orthogonal, choromanski2017the}.
    \item For Laplacian eigenvectors, we perform eigendecomposition of graph Laplacian matrix, \emph{i.e.}, rows of $\mathbf{U}$ from $\boldsymbol{\Delta}=\mathbf{I}-\mathbf{D}^{-1/2}\mathbf{A}\mathbf{D}^{-1/2}=\mathbf{U}^\top\boldsymbol{\Lambda}\mathbf{U}$, where $\mathbf{A}\in\mathbb{R}^{n\times n}$ is adjacency matrix, $\mathbf{D}$ is degree matrix, and $\boldsymbol{\Lambda}$, $\mathbf{U}$ correspond to eigenvalues and eigenvectors respectively~\cite{dwivedi2020benchmarking}.
\end{itemize}
The model expects $d_p$-dimensional node identifiers $\mathbf{P}\in\mathbb{R}^{n\times d_p}$, while ORF and Laplacian eigenvectors are $n$-dimensional.
To resolve this, if $n < d_p$, we zero-pad the channels.
If $n > d_p$, for ORF we randomly sample $d_p$ channels and discard the rest, and for Laplacian eigenvectors we use $d_p$ eigenvectors with the smallest eigenvalues following common practice~\cite{dwivedi2020benchmarking, lim2022sign, kreuzer2021rethinking}.

As the Laplacian eigenvectors are defined up to the factor $\pm1$ after normalized to unit length~\cite{dwivedi2020benchmarking}, we randomly flip their signs during training.
For PCQM4Mv2 (Section~\ref{sec:experiment_pcqm4mv2}), we apply random dropout on eigenvectors during training, similar to 2D channel dropout in ConvNets~\cite{tompson2015efficient}.
In our experiments with PCQM4Mv2, we find that both sign flip and eigenvector dropout work as effective regularizers and improves performance on validation data.

\subsubsection{Second-Order Equivariant Basis Approximation (Section~\ref{sec:experiment_synthetic})}\label{sec:apdx_experiment_details_synthetic}
\paragraph{Dataset}
\begin{table}[!t]
\caption{Statistics of the datasets.}
\label{table:dataset_statistics}
\begin{subtable}[h]{.5\linewidth}
    \caption{Statistics of Barab\'asi-Albert random graph dataset.}
    \centering
    \label{table:ba_dataset}
        \begin{tabular}{ll}
        \Xhline{2\arrayrulewidth}\\[-1em]
        \\[-1em] Dataset & Barab\'asi-Albert \\
        \\[-1em]\Xhline{2\arrayrulewidth}\\[-1em]
        \\[-1em] Size & 1280 \\
        \\[-1em] Average \# node & 14.9 \\
        \\[-1em] Average \# edge & 47.8 \\
        \\[-1em]\Xhline{2\arrayrulewidth}
    \end{tabular}
\end{subtable}\hfill
\begin{subtable}[h]{.5\linewidth}
    \caption{Statistics of PCQM4Mv2 dataset.}
    \centering
    \label{table:regression_dataset}
        \begin{tabular}{ll}
        \Xhline{2\arrayrulewidth}\\[-1em]
        \\[-1em] Dataset & PCQM4Mv2 \\
        \\[-1em]\Xhline{2\arrayrulewidth}\\[-1em]
        \\[-1em] Size & 3.7M \\
        \\[-1em] Average \# node & 14.1 \\
        \\[-1em] Average \# edge & 14.6 \\
        \\[-1em]\Xhline{2\arrayrulewidth}
    \end{tabular}
\end{subtable}
\end{table}
For the equivariant basis approximation experiment, we use a synthetic dataset containing Barab\'asi-Albert (BA) random graphs~\cite{barabasi1999emergence}.
With $\mathcal{U}$ denoting discrete uniform distribution, each graph is generated by first sampling the number of nodes $n\sim\mathcal{U}(10, 20)$ and the number for preferential attachment $k\sim\mathcal{U}(2, 3)$, then iteratively adding $n$ nodes by linking each new node to $k$ random previous nodes.
We do not utilize node or edge attributes and only use edge connectivity.
We generate 1152 graphs for training and 128 for testing.
Further dataset statistics is provided in Table~\ref{table:ba_dataset}.

\paragraph{Architecture}
Each model tested in Table~\ref{table:equivariant_basis_approximation_std} is a single multihead self-attention layer (Eq.~\eqref{eqn:transformer_residual_msa}) with hidden dimension $d=1024$, heads $H=\textnormal{bell}(2+2)=15$, and head dimension $d_H = 128$.
As for the node identifier dimension, we use $d_p=24$ for ORF and $d_p=20$ for Laplacian eigenvectors.

\paragraph{Experimental Setup}
We experiment with sparse or dense input graph representations.
For sparse input, we embed each graph with $n$ nodes and $m$ edges into $\mathbf{X}^{in}\in\mathbb{R}^{(n+m)\times(2d_p+d_e)}$.
For dense input, we use all $n^2$ pairwise edges and obtain $\mathbf{X}^{in}\in\mathbb{R}^{n^2\times(2d_p+d_e)}$, so that sparse edge connectivity is only used for obtaining Laplacian node identifiers.

For both sparse and dense inputs, we follow the standard procedure in Section~\ref{sec:methods} to use node and type identifiers to obtain $\mathbf{X}^{in}\in\mathbb{R}^{N\times(2d_p+d_e)}$ where $N=(n+m)$ or $n^2$, and project it to dimension $d$ with trainable projection $w^{in}$.
We also utilize a special token \texttt{[null]} with trainable embedding $\mathbf{X}_\texttt{[null]}\in\mathbb{R}^d$ (we shortly explain its use) to obtain the final input $[\mathbf{X}_\texttt{[null]};\mathbf{X}^{in}w^{in}]\in\mathbb{R}^{(1+N)\times d}$.

For an input $[\mathbf{X}_\texttt{[null]};\mathbf{X}^{in}w^{in}]$,
the goal is to supervise each of the $H=15$ self-attention heads with attention matrices $\boldsymbol{\alpha}^1,...,\boldsymbol{\alpha}^{15}$ to explicitly approximate row-normalized version of each equivariant basis $\mathbf{B}^{\mu_1}, ..., \mathbf{B}^{\mu_{15}}\in\mathbb{R}^{N\times N}$ (Eq.~\eqref{eqn:equivariant_linear_layer}) on the $N=(n+m)$ or $n^2$ input tokens (except \texttt{[null]}).
An issue in supervision is that for rows of $\mathbf{B}^\mu$ that only contains zeros, normalization is not defined.
To sidestep this, for such rows we simply supervise to attend to the special \texttt{[null]} token only.
For all other rows of $\mathbf{B}^\mu$ that contains nonzero entry, we supervise the model to ignore \texttt{[null]} token.

\paragraph{Training and Evaluation}
We train and evaluate all models with L2 loss between attention matrix $\boldsymbol{\alpha}^h$ and normalized basis tensor $\mathbf{B}^{\mu_h}$ (involving \texttt{[null]} token) averaged over heads $h=1,...,15$.
We train all models with AdamW optimizer~\cite{loshchilov2019decoupled} on 4 RTX 3090 GPUs each with 24GB.
For sparse inputs we use batch size 512, and for dense inputs we use batch size 256 due to increased memory cost.
We train all models for 3k steps (which takes about $\sim$1.5 hours) and apply linear learning rate warmup for 1k steps up to 1e-4 followed by linear decay to 0.
For all models, we use dropout rate of 0.1 on the input $[\mathbf{X}_\texttt{[null]};\mathbf{X}^{in}w^{in}]$ to prevent overfitting.

\subsubsection{Large-Scale Graph Learning (Section~\ref{sec:experiment_pcqm4mv2})}\label{sec:apdx_experiment_details_pcqm4mv2}
\paragraph{Dataset}
For large-scale learning, we use the PCQM4Mv2 quantum chemistry regression dataset from the OGB-LSC benchmark~\cite{hu2021ogb} that contains 3.7M molecular graphs.
Along with graph structure, we utilize both node and edge features \emph{e.g.}, atom and bond types following our standard procedure in Section~\ref{sec:methods}.
Dataset statistics is provided in Table~\ref{table:regression_dataset}.

\paragraph{Architecture}
All our models in Table~\ref{table:pcqm4mv2} (under \emph{Pure Transformers}) have the same encoder configuration following Graphormer~\cite{ying2021do}, with 12 layers, hidden dimension $d=768$, heads $H=32$, and head dimension $d_H=24$.
We adopt PreLN~\cite{xiong2020on} that places layer normalization before MSA layer (Eq.~\eqref{eqn:transformer_residual_msa}), MLP layer (Eq.~\eqref{eqn:transformer_residual_mlp}), and the final output projection after the last encoder layer.
We implement MLP (Eq.~\eqref{eqn:transformer_residual_mlp}) as a stack of two linear layers with GeLU nonlinearity~\cite{hendrycks2018gaussian} in between.
As for node identifier dimension, we use $d_p=64$ for ORF and $d_p=16$ for Laplacian eigenvectors.

As an additional GNN baseline, we run Graph Attention Network (GATv2)~\cite{velikovic2018graph, brody2022how} under several configurations.
For GAT and GAT-VN in Table~\ref{table:pcqm4mv2}, we use 5-layer GATv2 with hidden dimension 600 and a single attention head, having 6.7M parameters in total.
For GAT-VN (large), we use a 10-layer GATv2 with hidden dimension 1200 and a single attention head, having 55.2M parameters in total.
For GAT-VN and GAT-VN (large), we use virtual node that helps modeling global interaction~\cite{hu2021ogb}.

\paragraph{Training and Evaluation}
We mainly report and compare the Mean Absolute Error (MAE) on the validation data, and report MAE on the hidden test data if possible.
We train all models with L1 loss using AdamW optimizer~\cite{loshchilov2019decoupled} with gradient clipping at global norm 5.0.
We use batch size 1024 and train the models on 8 RTX 3090 GPUs with 24GB for $\sim$3 days.
We train our models for 1M iterations, and apply linear lr warmup for 60k iterations up to 2e-4 followed by linear decay to 0.
For our models in Table~\ref{table:pcqm4mv2} except TokenGT~(Lap)~+~Performer, we use the following regularizers:
\begin{itemize}
    \item Attention and MLP dropout rate 0.1
    \item Weight decay 0.1
    \item Stochastic depth~\cite{huang2016deep, steiner2021how} with linearly increasing layer drop rate, reaching 0.1 at last layer
    \item Eigenvector dropout rate 0.2 for TokenGT~(Lap) (see Appendix~\ref{sec:apdx_node_type_identifiers})
\end{itemize}

For TokenGT~(Lap)~+~Performer in Table~\ref{table:pcqm4mv2}, we load a trained model checkpoint of TokenGT~(Lap), change its self-attention to FAVOR+ kernelized attention of Performer~\cite{choromanski2021rethinking} that can provably accurately approximate softmax attention, and fine-tune it with AdamW optimizer for 0.1M training steps with 1k warmup iterations and cosine learning rate decay.
With batch size 1024 on 8 RTX 3090 GPUs, fine-tuning takes $\sim12$ hours.
We do not use stochastic depth and eigenvector dropout for fine-tuning.
For GAT baselines in Table~\ref{table:pcqm4mv2}, we use batch size 256 and train the models for 100 epochs with initial learning rate 0.001 decayed with a factor of 0.25 every 30 epochs.

\subsection{Additional Experimental Results (Section~\ref{sec:experiments})}\label{sec:apdx_additional_results}
\begin{figure}[!p]
\centering
    \rotatebox{270}{
        \begin{minipage}{0.99\textheight}
            \includegraphics[width=0.99\textheight]{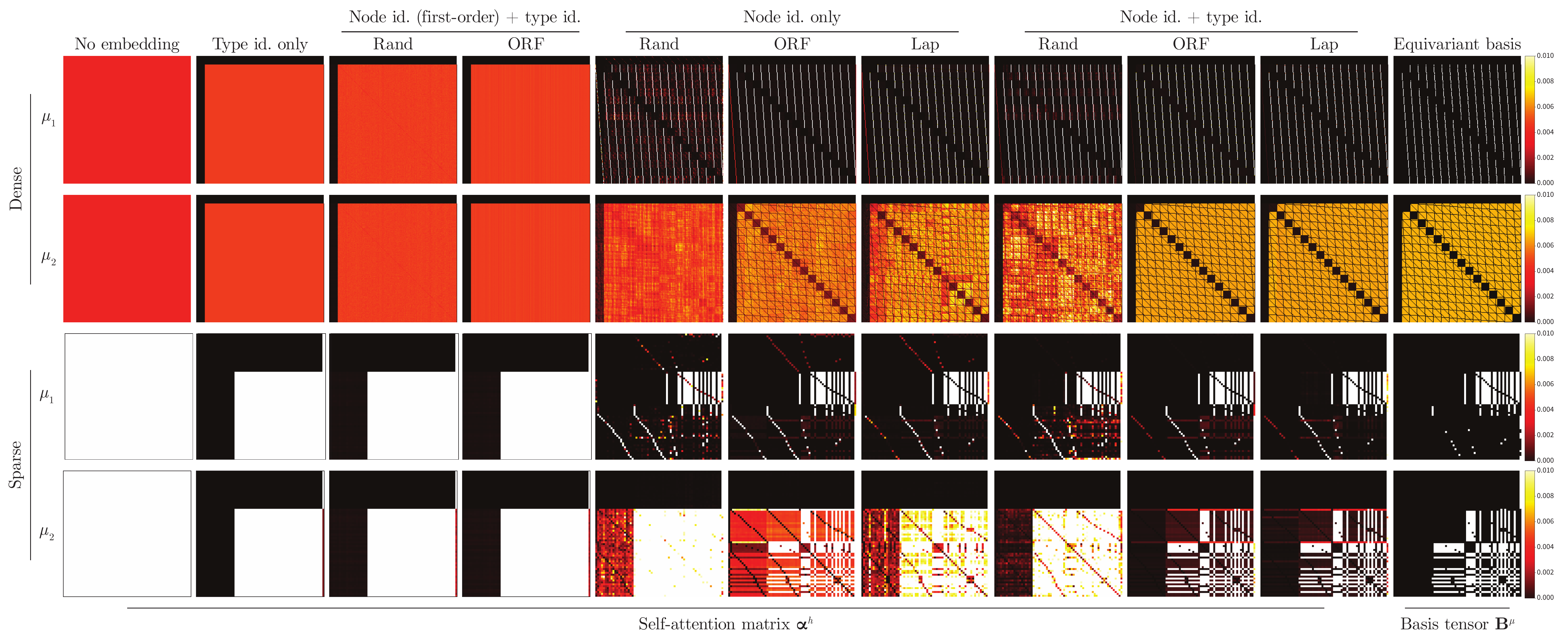}
            \caption{
            Self-attention maps learned under various node and type identifier configurations for two target equivariant basis tensors (out of 15), for both dense and sparse inputs.
            For better visualization, we clamp the entries by 0.01.
            Self-attention learns acute pattern coherent to equivariant basis when orthonormal node identifiers and type identifiers are provided both as input.
            }\label{fig:self_attention_maps_more}
        \end{minipage}
    }
\end{figure}
\begin{figure}[!t]
    \centering
    \includegraphics[width=0.7\textwidth]{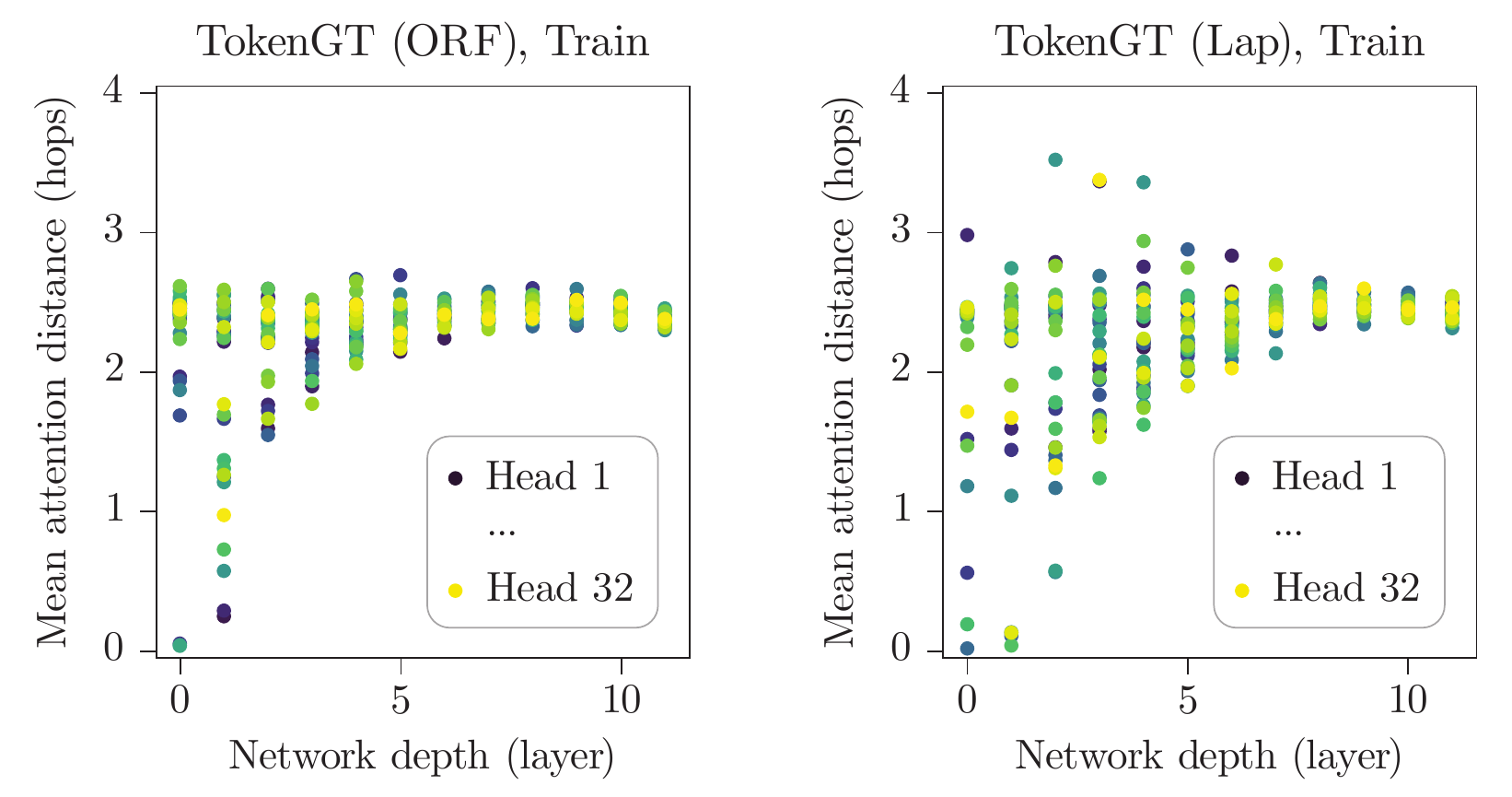}
    \vspace{-0.2cm}
    \caption{
    Attention distance by head and network depth, measured for entire PCQM4Mv2 training set.
    Each dot shows mean attention distance in hops across graphs of a head at a layer.
    }
    \label{fig:self_attention_distance_more}
\vspace{-0.2cm}
\end{figure}
We report additional experimental results and discussions that could not be included in the main text due to space restriction.

\subsubsection{Second-Order Equivariant Basis Approximation (Section~\ref{sec:experiment_synthetic})}\label{sec:apdx_additional_results_synthetic}
In addition to the Figure~\ref{fig:self_attention_maps} in the main text that shows learned self-attention maps for dense input, in Figure~\ref{fig:self_attention_maps_more}, we provide an extended visualization of self-attention maps for both dense and sparse inputs.
Consistent to Lemma~\ref{lemma:approximation_equivariant_basis} and Table~\ref{table:equivariant_basis_approximation_std}, self-attention achieves accurate approximation of equivariant basis only when both the orthonormal node identifiers (ORF or Lap) and type identifiers are given.

\subsubsection{Large-Scale Graph Learning (Section~\ref{sec:experiment_pcqm4mv2})}\label{sec:apdx_additional_results_pcqm4mv2}
In addition to the Figure~\ref{fig:self_attention_distance} in the main text that shows attention distance measured for the PCQM4Mv2 validation data, in Figure~\ref{fig:self_attention_distance_more}, we provide an extended figure of attention distance measured for the entire training set that contains $\sim$3M graphs.
Overall we find similar trends as analyzed in Section~\ref{sec:experiment_pcqm4mv2}.

\subsubsection{Transductive Node Classification on Large Graphs (Section~\ref{sec:experiments})}\label{sec:apdx_additional_results_transductive_node_classification}
%While we acknowledge that our empirical evaluation is conducted only on PCQM4Mv2, we would like to note that PCQM4Mv2 is one of the largest-scale graph dataset up to date containing 3.8M graphs~\cite{hu2021ogb}. This makes it one of the few suitable benchmark to test our model as Transformers are generally designed to work with extremely large-scale data~\cite{brown2020language, dosovitskiy2021animage}.
\begin{table}[!t]
\caption{Statistics of the transductive node classification datasets.}
\label{table:dataset_statistics_transductive}
\centering
\label{table:transductive_dataset}
\begin{adjustbox}{width=0.7\textwidth}
\begin{tabular}{lllllll}
    \Xhline{2\arrayrulewidth}\\[-1em]
    \\[-1em] Dataset & CS & Physics & Photo & Computers & Chameleon & Crocodile \\
    \\[-1em]\Xhline{2\arrayrulewidth}\\[-1em]
    \\[-1em] \# nodes & 18,333 & 34,493 & 7,650 & 13,752 & 2,277 & 11,631 \\
    \\[-1em] \# edges & 81,894 & 247,962 & 119,081 & 245,861 & 36,101 & 180,020 \\
    \\[-1em] \# classes & 15 & 5 & 8 & 10 & 6 & 6 \\
    \\[-1em]\Xhline{2\arrayrulewidth}
\end{tabular}
\end{adjustbox}
\end{table}
\begin{table}[!t]
\vspace{-0.2cm}
\caption{
Transductive node classification.
OOM denotes out-of-memory error on a 24GB RTX 3080 GPU.
We report aggregated test accuracy at best validation accuracy over 7 randomized runs.
}
\vspace{-0.2cm}
\centering
\begin{adjustbox}{width=0.99\textwidth}
\begin{tabular}{lcccccc}\label{table:transductive_node_classification}
    \\\Xhline{2\arrayrulewidth}\\[-1em]
    & CS & Physics & Photo & Computers & Chameleon & Crocodile \\
    \Xhline{2\arrayrulewidth}\\[-1em]
    GCN & 0.895 $\pm$ 0.004 & 0.932 $\pm$ 0.004 & 0.926 $\pm$ 0.008 & 0.873 $\pm$ 0.004 & 0.593 $\pm$ 0.01 & 0.660 $\pm$ 0.01 \\
    GAT & 0.893 $\pm$ 0.005 & 0.937 $\pm$ 0.01 & 0.947 $\pm$ 0.006 & \textbf{0.914 $\pm$ 0.002} & 0.632 $\pm$ 0.011 & 0.692 $\pm$ 0.017 \\
    GIN & 0.895 $\pm$ 0.005 & 0.886 $\pm$ 0.046 & 0.886 $\pm$ 0.017 & 0.362 $\pm$ 0.051 & 0.479 $\pm$ 0.027 & 0.515 $\pm$ 0.041 \\
    Graphormer & 0.791 $\pm$ 0.015 & OOM & 0.894 $\pm$ 0.004 & 0.814 $\pm$ 0.013 & 0.457 $\pm$ 0.011 & 0.489 $\pm$ 0.014 \\
    \Xhline{2\arrayrulewidth}\\[-1em]
    TokenGT (Near-ORF) + Performer & 0.882 $\pm$ 0.007 & 0.931 $\pm$ 0.009 & 0.872 $\pm$ 0.011 & 0.82 $\pm$ 0.019 & 0.568 $\pm$ 0.019 & 0.583 $\pm$ 0.024 \\
    TokenGT (Lap) + Performer & 0.902 $\pm$ 0.004 & 0.941 $\pm$ 0.007 & 0.919 $\pm$ 0.009 & 0.86 $\pm$ 0.012 & 0.637 $\pm$ 0.032 & 0.638 $\pm$ 0.025 \\
    TokenGT (Lap) + Performer + SEB & \textbf{0.903 $\pm$ 0.004} & \textbf{0.950 $\pm$ 0.003} & \textbf{0.949 $\pm$ 0.007} & 0.912 $\pm$ 0.006 & \textbf{0.653 $\pm$ 0.029} & \textbf{0.718 $\pm$ 0.012} \\
    \Xhline{2\arrayrulewidth}
\end{tabular}
\end{adjustbox}
\vspace{-0.35cm}
\end{table}
While our main experiment in Section~\ref{sec:experiment_pcqm4mv2} focuses on graph-level predictions, TokenGT can in principle be applied to a more broad class of node-level or edge-level graph understanding tasks by putting prediction head on appropriate output tokens.
To demonstrate this, we conduct additional experiments on a variety of transductive node classification datasets.
In contrast to PCQM4Mv2, they involve large graphs with up to tens of thousands of nodes, posing a challenge to $\mathcal{O}(n^2)$ complexity methods such as graph Transformers that rely on dense attention bias.

\paragraph{Dataset} We use transductive node classification datasets, where each data is represented as a node in a large-scale graph, including co-authorship (CS, Physics)~\cite{schur2018pitfalls}, co-purchase (Photo, Computers)~\cite{schur2018pitfalls}, and Wikipedia page networks (Chameleon, Crocodile)~\cite{rozemberczki2021multi}.
We randomly split the dataset into train, validation, and test sets by randomly reserving 30 random nodes per class for validation and test respectively, and use the rest of the nodes for training.
Dataset statistics is provided in Table~\ref{table:dataset_statistics_transductive}.

\paragraph{Approach}
We utilize simple variants of TokenGT with Performer kernel attention of $\mathcal{O}(n+m)$ complexity.
Due to the large number of nodes $n$, an immediate challenge for TokenGT is dealing with the orthonormality assumption on the node identifiers (Lemma~\ref{lemma:approximation_equivariant_basis}) as the maximal number of orthonormal node identifiers is bounded by dimension $d_p$.
In this case, it is reasonable to introduce \emph{near-orthonormal} vectors as node identifiers, as it is theoretically guaranteed that we can draw an exponential number $\mathcal{O}(e^{\Omega(d_p)})$ of $d_p$-dimensional near-orthonormal vectors~\cite{gorban2016approximation}.
For \emph{TokenGT (Near-ORF)}, we use $d_p = 64$-dimensional random node identifiers where each entry is sampled from $\{-1/d_p,+1/d_p\}$ with coin toss~\cite{gorban2016approximation}.
For \emph{TokenGT (Lap)}, we use a subset of the Laplacian eigenvectors as node identifiers, specifically $d_p/2$ eigenvectors with lowest eigenvalues and $d_p/2$ eigenvectors with highest eigenvalues, and choose $d_p$ among $64$-$100$ based on validation performance.

While \emph{Near-ORF} and \emph{Lap} can theoretically serve as an efficient low-rank approximation for orthonormal node identifiers, their approximation can affect the quality of modeled equivariant basis (Section~\ref{sec:theory}).
In particular, equivariant basis ($\mu$) represented as \textbf{sparse} basis tensor ($\mathbf{B}^\mu$; Definition~\ref{defn:basis_tensor}) are expected to be affected more, as they require most entries to be zero.
To remedy this, we take a simple approach of residually adding one of such sparse equivariant operators $\mathbf{X}_{ii}\mapsto\mathbf{X}_{ii} + \sum_{j\neq i}\mathbf{X}_{ij}$ explicitly after each Transformer layer.
We denote this variant as \emph{TokenGT (Lap) + Performer + SEB}, where SEB abbreviates sparse equivariant basis.
This fix is minimal, easy to implement, and highly efficient as it only requires a single \texttt{torch.coalesce()} call, and also empirically effective.

\paragraph{Architecture}
All our models in Table~\ref{table:transductive_node_classification} utilize a linear prediction head on the node tokens obtained at the final Transformer layer to perform node-level classification.
We perform an exploratory hyperparameter search over the number of layers from $2$-$4$, heads $H$ from $1$-$4$, hidden dimension $d$ from $128$-$1024$, and dropout rate from $\{0.1, 0.5\}$, based on validation performance.

We employ strong message-passing GNN and graph Transformer baselines, including GCN~\cite{kipf2017semi}, GAT~\cite{velikovic2018graph}, GIN~\cite{xu2019how} which has 2-WL expressiveness similar to ours, and Graphormer~\cite{ying2021do} based on fully-connected node self-attention.
For message-passing GNNs, we use a 4-layer architecture and search hidden dimension $d$ from $\{64, 1024\}$ based on validation performance.
For Graphormer, we perform an exploratory search on the number of layers from $1$-$4$, heads $H$ from $1$-$4$, and hidden dimension $d$ from $128$-$1024$ based on validation performance.
We apply 0.5 dropout for all baselines.

\paragraph{Training and Evaluation}
We report and compare classification accuracy on the test nodes at best validation accuracy aggregated over 7 randomized runs.
We train all models with node-level categorical cross-entropy loss using Adam optimizer~\cite{kingma2015adam} on a single RTX 3090 GPU with 24GB.
We train all models with a learning rate of 1e-3 for 300 epochs.

\paragraph{Results}
The results are in Table~\ref{table:transductive_node_classification}. Graphormer~\cite{ying2021do} suffers out-of-memory in the Physics dataset mainly due to the spatial encoding that requires $\mathcal{O}(n^2)$ memory.
By constraining the model capacity appropriately, we were able to run Graphormer on other datasets.
However, we observe a low performance, presumably due to the memory cost that prevents depth and head scaling.
As the spatial encoding is incorporated into the model via attention bias, the model strictly requires $\mathcal{O}(n^2)$ memory and cannot be easily made more efficient.
On the other hand, TokenGT variants are able to utilize Performer attention with $\mathcal{O}(m+n)$ cost, which allows using larger models to achieve the best performance in all but one dataset (Computers, where the performance is on par with the best model).

\subsection{Additional Discussion on Performance on PCQM4Mv2 (Section~\ref{sec:experiment_pcqm4mv2})}\label{sec:apdx_extended_discussion_pcqm4mv2}
As in the Table~\ref{table:pcqm4mv2} in the main text, TokenGT currently shows a slightly lower performance compared to the Graphormer and its successors in the PCQM4Mv2 benchmark.
We conjecture this is partly because we intentionally keep its components simple to faithfully adhere to the equivariance theory.
We discuss some engineering approaches that may enhance the performance of TokenGT \emph{at the cost of differentiating from the theory}.
We consider engineering TokenGT to match or outperform sophisticated graph Transformers as a promising and important next research direction.

\paragraph{Node Identifiers} Our best performing TokenGT (Lap) currently uses Laplacian eigenvectors~\cite{dwivedi2020benchmarking} as the node identifiers, which has been criticized for issues such as loss of structural information~\cite{kreuzer2021rethinking} and sign ambiguity~\cite{lim2022sign}.
Thus, one could try to relax the theoretical requirement for orthonormality of node identifiers and incorporate more powerful node positional encodings~\cite{kreuzer2021rethinking, lim2022sign} as node identifiers, which could potentially yield better performance in practice.

\paragraph{(Hyper)Edge Tokens} TokenGT currently treats an undirected input edge $(u, v)$ as if both directions $(u, v)$ and $(v, u)$ are present, leading to a pair of edge tokens $[\mathbf{X}_{(u,v)}, \mathbf{P}_u, \mathbf{P}_v]$ and $[\mathbf{X}_{(v,u)}, \mathbf{P}_v, \mathbf{P}_u]$.
Similarly, an undirected order-$k$ input hyperedge $(v_1, ..., v_k)$ of an higher-order hypergraph is parsed to all possible orderings of node identifiers.
While this is a common characteristic of tensor-based permutation equivariant neural networks~\cite{maron2019invariant,maron2019on,keriven2019universal,serviansky2020set,maron2019provably,kim2021transformers}, they can lead to memory overhead and redundancy since multiple tokens represent an identical undirected edge.
To avoid this, one can use a single token for each undirected (hyper)edge and pool the node identifiers as $\sum_{i=1}^k\rho(\mathbf{P}_{v_i})$.
Combined with powerful node identifiers, this approach could potentially enhance the model performance.

% camera ready
% \newpage
% \input{checklist}
% \newpage
% \pagenumbering{arabic}
% \input{appendix}

\end{document}